\documentclass{article}
\usepackage[a4paper]{geometry}
\usepackage{hyperref}
\usepackage[numbers]{natbib}
\usepackage{comment}
\usepackage{graphicx}
\usepackage{booktabs} %

\usepackage{hyperref}
\usepackage{dsfont}
\usepackage{amsmath}

\usepackage{algorithm}
\usepackage{algorithmic}

\usepackage{amsmath}
\usepackage{amssymb}
\usepackage{mathtools}
\usepackage{amsthm}
\usepackage{verbatim}
\usepackage{float}
\usepackage{multirow}
\usepackage[table]{xcolor}
\usepackage{subcaption}
\usepackage{tcolorbox}
\tcbuselibrary{listings,breakable}

\usepackage[capitalize,noabbrev]{cleveref}

\theoremstyle{plain}
\newtheorem{theorem}{Theorem}[section]

\newtheorem{lemma}[theorem]{Lemma}

\theoremstyle{definition}
\newtheorem{definition}[theorem]{Definition}

\theoremstyle{remark}

\DeclareMathOperator*{\argmax}{arg\,max}

\usepackage[disable,textsize=tiny]{todonotes}
\usepackage[textsize=tiny]{todonotes}

\date{}

\title{ScoreFlow: Mastering LLM Agent Workflows via\\Score-based Preference Optimization}

\author{
  Yinjie Wang\textsuperscript{1,*}~~
  Ling Yang\textsuperscript{2,*}~~
  Guohao Li\textsuperscript{3}~~
  Mengdi Wang\textsuperscript{2}~~
  Bryon Aragam\textsuperscript{1}\\[6pt]
  Project: \url{https://github.com/Gen-Verse/ScoreFlow}
}

\begin{document}
\maketitle

\let\thefootnote\relax

\footnotetext{%
\textsuperscript{*}Equal contribution.
\textsuperscript{1}University of Chicago.\quad
\textsuperscript{2}Princeton University.\quad
\textsuperscript{3}University of Oxford.\quad
Correspondence to:
yangling0818@163.com,
yinjie@uchicago.edu.
}

\begin{abstract}
Recent research has leveraged large language model multi-agent systems for complex problem-solving while trying to reduce the manual effort required to build them, driving the development of automated agent workflow optimization methods. However, existing methods remain inflexible due to representational limitations, a lack of adaptability, and poor scalability when relying on discrete optimization techniques. We address these challenges with ScoreFlow, a simple yet high-performance framework that leverages efficient gradient-based optimization in a continuous space. ScoreFlow incorporates Score-DPO, a novel variant of the direct preference optimization method that accounts for quantitative feedback. Across six benchmarks spanning question answering, coding, and mathematical reasoning, ScoreFlow achieves an 8.2\% improvement over existing baselines. Moreover, it empowers smaller models to outperform larger ones with lower inference costs.
\end{abstract}

\section{Introduction}

Large language models (LLMs) have demonstrated proficiency in solving natural language tasks \citep{ouyang2022training, touvron2023llama, achiam2023gpt, anil2023palm, yang2024buffer, yang2024supercorrect}. Furthermore, the multiagent system (workflow) of LLMs, where multiple agents coordinate and exchange information to complete tasks, enables LLM-based agents to collaborate and solve complex tasks across a wide range of domains, such as mathematical problem solving \citep{zhong2024achieving, xu2023lemur}, question answering \citep{nori2023can}, and coding tasks \citep{hong2024data, ridnik2024code}.

These manually designed agentic workflows, however, require significant effort and have limited capacity to handle tasks across diverse domains. Therefore, the emerging focus in this area is to address the limitations of static workflows by developing automated methods for workflow generation and optimization. These optimizations can target various aspects, including prompt refinement, hyperparameter tuning, and workflow structure design \citep{khattab2024dspy, zhuge2023mindstorms, yuksekgonul2024textgrad, hu2024automated, zhang2024aflow, chen2023autoagents, li2024autoflow, liu2024dynamic, song2024adaptive, zhang2024g}.

The automated optimization methods can be constrained by the limitations inherent in pre-defined workflow structures and the rigidity of workflow space representations \citep{khattab2024dspy, zhuge2023mindstorms, yuksekgonul2024textgrad, liu2024dynamic}.
DyLAN \citep{liu2024dynamic} thoughtfully emphasizes the communication structure within LLM debates but overlooks other potential communication structures. GPTSwarm \citep{zhuge2023mindstorms} leverages graph-based structures and employs reinforcement fine-tuning for optimization. However, the lack of consideration for conditional states within the graph structure imposes restrictions on the search space.

To improve representation capabilities, AFlow \citep{zhang2024aflow} and ADAS \citep{hu2024automated} employ code as representation for workflow, facilitating robust and flexible workflow searches. However, ADAS faces challenges with inefficient search processes and coarse workflow storage, which leads to the accumulation of irrelevant data and increased complexity, ultimately reducing its effectiveness. To address these issues, AFlow employs a variant of the Monte Carlo Tree Search as an optimization method to enhance efficiency. However, the overly rapid convergence on workflow structures, combined with the discrete optimization method, restricts the exploration of the search space, often leading to suboptimal outcomes. Additionally, they all optimize a single workflow for the entire task set, which limits adaptability and scalability for larger datasets containing diverse problems \citep{zhang2024g, song2024adaptive}.

To address these challenges, we propose \textbf{ScoreFlow}, an automated and cost-efficient multi-agent workflow generation framework that employs a novel optimization method to achieve high performance, scalability, and adaptability. For each given task, the workflow generator constructs its workflow using code as a representation and the generator is further optimized based on evaluation scores feedback. The loss-gradient optimization makes it more flexible and scalable than previous discrete optimization methods \citep{zhang2024aflow, hu2024automated, liu2024dynamic}. Furthermore, by leveraging an open-source LLM as the foundational model for workflow generation, our framework minimizes the costs associated with workflow generation. This approach addresses the challenge of high API call expenses inherent in the workflow generation process \citep{hu2024automated, song2024adaptive}.

In the optimization process, we collect preference pairs from evaluation scores to construct preference data, which are subsequently used to fine-tune the workflow generator via a novel variant of direct preference optimization (DPO) \citep{rafailov2024direct}. While DPO is efficient and stable, variance and inaccuracies in evaluation scores reduce the reliability of preference data, slowing convergence and hindering optimal performance within limited iterations. To address these limitations, we propose a widely applicable preference optimization method, \textbf{Score-DPO}, which incorporates quantitative score information directly into the optimization process.

We highlight our following contributions:

\setlength{\parskip}{0pt}
\begin{itemize}
\setlength{\itemsep}{1pt} %
    \setlength{\parskip}{1pt} %
    \item \textbf{ScoreFlow:} We introduce ScoreFlow, a simple yet flexible, automated, and adaptive framework for agentic workflow generation and optimization, minimizing the need for human intervention.
    \item \textbf{Score-DPO:} We propose Score-DPO, an optimization method that can be broadly applied in similar settings, leveraging quantitative evaluation feedback rather than relying solely on preference pairs by integrating evaluation scores into the preference optimization process. Its effectiveness is demonstrated through both experimental results and theoretical analysis.
    \item \textbf{Extensive Evaluations:} We evaluate ScoreFlow with Score-DPO on six benchmark datasets across three diverse tasks: question answering, coding, and mathematical reasoning. Our approach outperforms baseline methods by 8.2\%. Extensive studies further highlight the robustness, scalability, and cost-efficiency of ScoreFlow across different models and reveal its ability to enable smaller models to surpass larger models in performance while achieving greater cost efficiency.
\end{itemize}

\section{Related Work}

\subsection{Agentic Workflow Optimization}

\paragraph{Automated Optimizations for Prompt and Hyperparameter} 
 
Automated optimization methods emphasizing prompt optimization \citep{fernando2023promptbreeder, yuksekgonul2024textgrad, yang2023large, khattab2024dspy} or hyperparameter optimization \citep{saad2024archon} can enhance performance; however, they impose limitations on the workflow structure and often require manual modifications to accommodate new tasks, restricting their adaptability and scalability.

\paragraph{Automated Optimizations for Workflow Structure} Workflow optimization methods \citep{zhou2024symbolic, zhuge2023mindstorms, hu2024automated, zhang2024aflow, chen2023autoagents, li2024autoflow, liu2024dynamic, song2024adaptive, zhang2024g} focus on refining the structure of workflows, making them more robust for handling diverse tasks.
However, the inflexibility and limitations in workflow representation, such as the loss of conditional states within the graph structure, may restrict the search space and consequently hinder the ability to accommodate diverse and complex workflows.
To address this challenge, ADAS \citep{hu2024automated} and Aflow \citep{zhang2024aflow} adopt code as a representation for workflows.
However, the performance of ADAS is constrained by its accumulated irrelevant information and increased complexity in optimization, hindering agents' ability. Aflow employs a Monte Carlo Tree Search-based method to efficiently identify optimal workflows; however, its tendency toward premature convergence on workflow structures limits the exploration of the search space. Moreover, the discrete optimization method, which involves randomly selecting failed cases and feeding them back to the optimizer LLM to refine the workflow, imposes significant limitations on scalability.

\subsection{Learning from Preferences for Language Models}

\paragraph{PPO} Proximal Policy Optimization (PPO) \citep{schulman2017proximal} process preference feedback in two stages. First, a reward model \( R_{\phi} \) is trained on the preference dataset \( D_{R} \), where each entry \( (x, y_w, y_l) \) consists of a prompt \( x \), a preferred response \( y_w \), and a rejected response \( y_l \). The reward model is optimized by minimizing the following loss function, which is inspired by the Bradley-Terry (BT) model \citep{bradley1952rank} for pairwise ranking:
\begin{align}
\label{pporewardobj}
- \mathbb{E}_{(x, y_w, y_l) \sim D_{R}} \big[\log \sigma \big(R_{\phi}(x, y_w) - R_{\phi}(x, y_l)\big)\big].
\end{align}
Next, the policy model \( \pi_{\theta} \) is refined by maximizing the reward assigned to its generated responses, while maintaining a soft KL divergence constraint to prevent degeneration. The objective is expressed as:
\begin{align}
\label{ppopolicyobj}
\mathbb{E}_{x \sim D_{\pi}, y \sim \pi_{\theta}(y \mid x)} \big[R_{\phi}(x, y)\big] - \beta \mathbb{D}_{KL} (\pi_{\theta} || \pi_{ref}),
\end{align}
where \( \pi_{ref} \) represents the reference policy, and \( \beta \) is a hyperparameter controlling the KL penalty.

\paragraph{DPO} Direct Preference Optimization (DPO) \citep{rafailov2024direct} facilitates direct policy optimization using preference data, eliminating the need for explicit reward models or active policy sampling. This approach enhances both the efficiency and stability of the optimization process. From the closed-form solution of Equation~\ref{ppopolicyobj}, the implicit reward can be expressed as \( R_{\phi}(x, y) =  \beta \log \big(\pi_{\theta^{\star}}(y \mid x) / \pi_{ref}(y \mid x)\big) + \beta Z(x)\), where \( \pi_{\theta^{\star}} \) is the optimal policy and $Z(x)$ is a partition function. The policy model can then be directly optimized using the reward objective in Equation~\ref{pporewardobj}, resulting in the DPO loss:
\[
- \mathbb{E}_{(x, y_w, y_l) \sim D_{R}} \big[\log \sigma \big(r(x, y_w) - r(x, y_l)\big)\big],
\]
where $r(x, y) := \beta \log \big(\pi_{\theta^{\star}}(y \mid x) / \pi_{ref}(y \mid x)\big)$.

When the data format includes associated evaluation scores for each sample (as in our setting), rather than solely chosen and rejected pairs, we propose Score-DPO, which integrates these scores into the training process to enhance performance. This approach achieves improved performance over standard DPO while maintaining its efficiency and stability in our applications.

\begin{figure*}[tp]
    \centering
    \includegraphics[width=\textwidth]{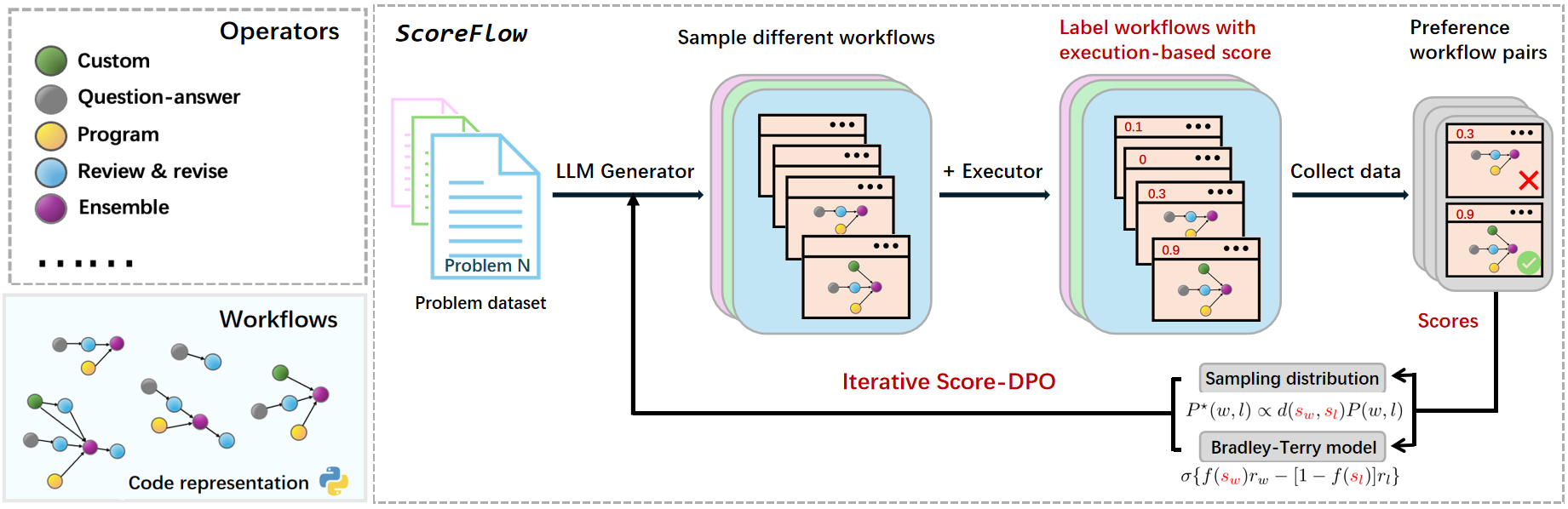}
    \caption{\textbf{Pipeline of ScoreFlow.} First, for each problem in the dataset, multiple workflows are generated. Next, an executor is employed to execute these workflows for corresponding problems, resulting in evaluation scores. Based on these scores, preference data is collected. Subsequently, incorporating the score information, the Score-DPO algorithm is used to fine-tune the generator. This process is iterated until the maximum number of iterations is reached or convergence is achieved. }
    \label{pipline}
\vspace{-3mm}
\end{figure*}

\section{ScoreFlow}

\subsection{Background}

We provide a preliminary overview of ScoreFlow's inference process, as illustrated in Figure~\ref{figspecificeample}. Given Math tasks A and B, along with selectable agent types—programmer, customizable operator, ensemble operator, and reviewer—a Python-based workflow is generated for each task, where the agent sets of workflows A and B contain one and five agents, respectively. Each task is then input into its respective workflow to produce the executed result.

Now we formalize the LLM multi-agent workflow optimization problem and some notations as follows. Given an input task \( q \), formatted as a prompt, we want to determine the optimal workflow \( G(q) \) to solve this task, where \( G \) is the workflow generator. A workflow function \( W_f \) is defined as a mapping that maps the integration of some task \( q \) and the agent set \( V \), \((q, V)\), to executed results \( W_f(q, V) \), typically the solution to the task. The agent set \( V \) consists of a collection of agents, each characterized by their system prompts, temperature settings, and other relevant parameters. Then, the \textbf{workflow} is defined as the combination of an agent set and a workflow function: \( (V, W_f) \).
We define the \textbf{workflow search space} as: $\mathcal{W} = \{ (V, W_f) \mid V \subset \mathcal{V}, (V, W_f) \text{ satisfies the condition } C \}$, where \( \mathcal{V} \) represents the whole agent space. The condition \( C \) imposes constraints on the search space, such that \( W_f \) is executable for the agent set \( V \).
Given these notations, our optimization objective is to identify the optimal workflow generator:
\begin{align*}
G^{\star} = \argmax_{G: \operatorname{Im}(G) \subset \mathcal{W}} \mathbb{E}_{q \in D} \big[ S(q, G(q)) \big],
\end{align*}
where \( D \) represents the dataset of tasks, and \( S \) is a third-party evaluator for the result generated by executing the workflow \( G(q) \) on task \( q \), such as a human-provided score, the average win rate, or other relevant metrics.

Using code as a representation of the workflow function \( W_f \) \citep{hu2024automated, zhang2024aflow} can account for linear sequences, loops, conditional logic, and provide flexibility that exceeds graph or network structures. Furthermore, following Aflow \citep{zhang2024aflow}, we characterize agents in \( \mathcal{V} \) as operators. The operators are predefined, reusable combinations of agent nodes representing common operations, such as programmers, reviewers, revisers, question-answering operators, ensemble operators, test operators and customizable operators, etc. By allowing the system prompts within operators to be customizable by the generator \( G \), we achieve optimization for the prompt, expand the operator space \( \mathcal{V} \), and enrich the search space \( \mathcal{W} \).

To make the workflow adaptive for the input task \( q \), that is, to adapt the chosen operators and the structural complexity of the generated workflow according to the input problem, it is necessary to extract semantic information from \( q \). Specifically, we use an open-source pre-trained large language model as the base model for our generator \( G \). The input to the generator consists of the combination of the task \( q \) and guidance on generation, including format requirements and introductions to available operators, all formatted as a guidance prompt. The detailed guidance prompt is provided in the Appendix~\ref{appdetailedscoreflowmethods}.

\begin{figure*}[t]
    \centering
    \begin{subfigure}[t]{0.45\textwidth}
        \centering
        \includegraphics[height=5.5cm]{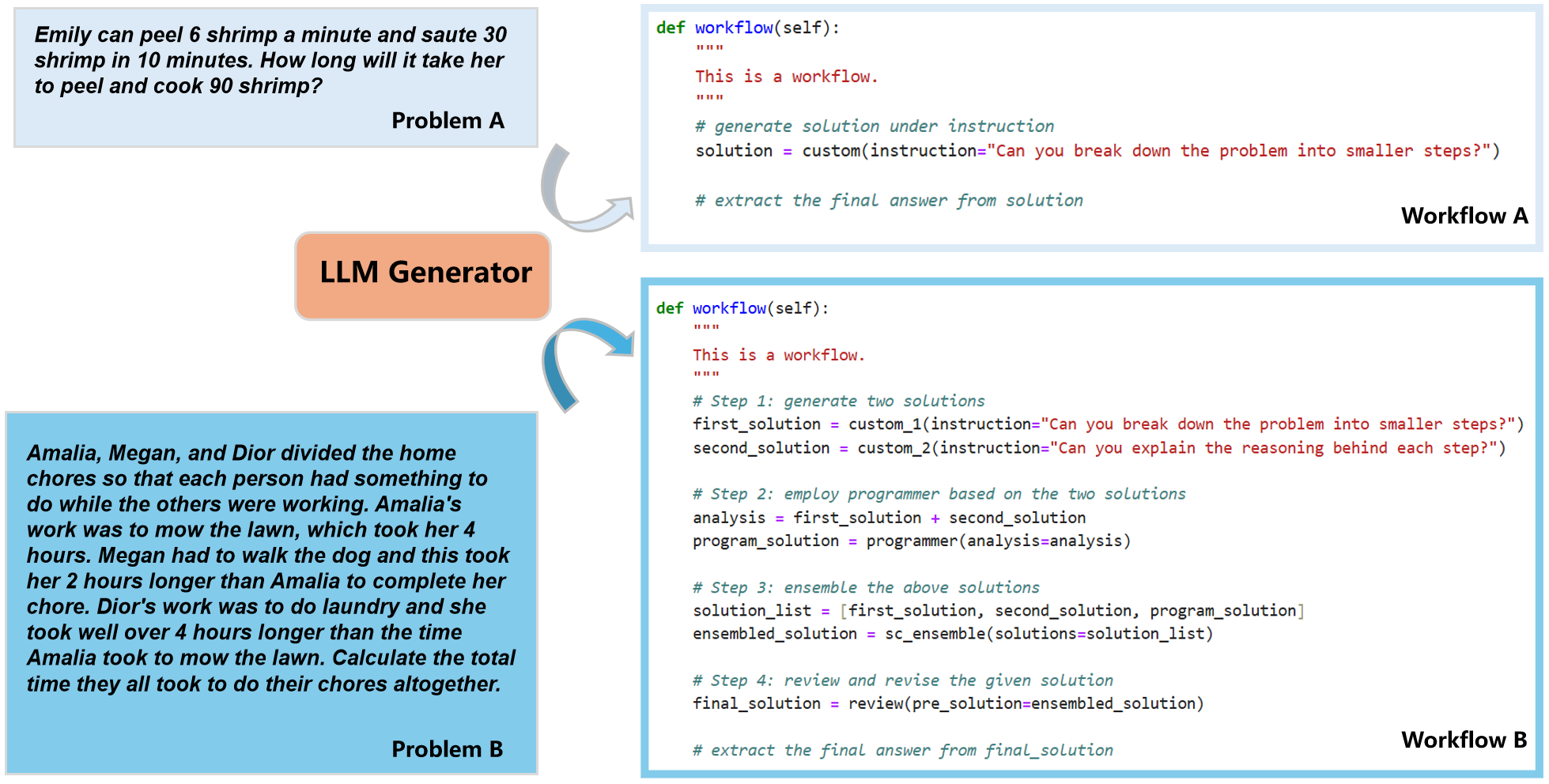}
    \end{subfigure}
    \hfill
    \begin{subfigure}[t]{0.32\textwidth}
        \centering
        \includegraphics[height=5.5cm]{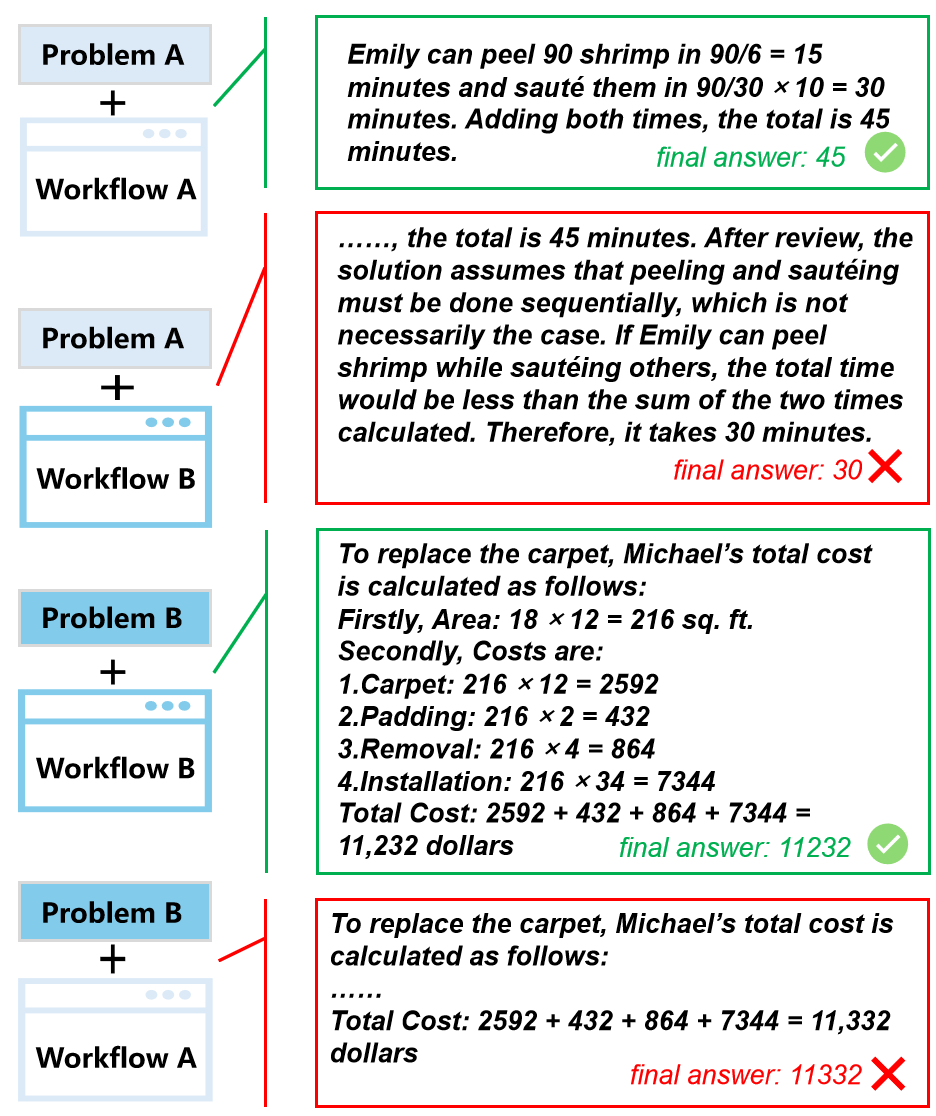}
    \end{subfigure}
    \caption{Illustration of the inference process: Two distinct workflows are generated for two GSM8K problems, and their executed results are evaluated. The executor utilized is GPT-4o-mini, with a temperature of 0. This plot highlights the adaptivity of the generation process.}
    \label{figspecificeample}
\vspace{-5mm}
\end{figure*}

\subsection{ScoreFlow Overview}
\label{secscoreflow}

In this section, we provide a high-level overview of our proposed method, \textbf{ScoreFlow}, as illustrated in Figure~\ref{pipline}, while deferring the detailed procedural steps to Sections~\ref{secqlopw} and \ref{secscore-dpo}. At each iteration, we start by collecting preference data. For each task, we generate multiple workflows using the generator $G$, evaluate the execution results to obtain evaluation scores, and derive preference pairs based on these scores. To optimize the generator \(G\) using this preference dataset, we propose \textbf{Score-DPO}, an enhanced version of Direct Preference Optimization (DPO) \citep{rafailov2024direct}. The generator is fine-tuned on preference dataset, and the updated generator is employed in the subsequent iteration.
The iterative process stops when it achieves convergence or reaches the maximum iteration number \( M \).

\subsection{Quantitative Labeling of Preference Workflows}
\label{secqlopw}

In this section, we explain the process of assigning quantitative labels and collecting preference workflow data.
We generate \(k\) workflows for each task \(q\), denoted as \(g_i(q)\), where \(1 \leq i \leq k\). Each workflow \( g_i(q) \) produces execution results for the corresponding \( q \), which are subsequently evaluated to derive an associated evaluation score, denoted as \( s_i \), where \( s_i \in [0,1] \).
The evaluation score is derived by using an independent executor LLM to execute the workflow, calculating the average F1 score or win rate from their outputs in our experiments. Unlike self-improvement methods, which rely on the generator for evaluation \citep{jiang2024self}, thereby making the iteration self-referential, our approach uses third-party sources (e.g., validation datasets and executor LLM).
Next, we construct preference pairs for task problem \(q\) in the form \(D_q = \{(q, g_i(q)), (q, g_j(q)) ) \mid s_i > s_j\}\), which are then aggregated to form the complete preference dataset, \(D_{pre} = \bigcup_{q \in D} D_q\). For simplicity, we denote each element in \( D_{pre} \) as \((w, l)\), where the winner \( w \) includes the prompt \( x \), chosen workflow \( y_w \), and evaluation score \( s_w \), while the loser \( l \) includes \( x \), rejected workflow \( y_l \), and score \( s_l \).

\subsection{Optimization via Score-DPO}
\label{secscore-dpo}

We observe that directly using DPO to finetune the generator on collected preference data results in slow convergence and an inability to achieve optimal performance. These issues are due to errors and variance in the evaluation scores. We propose a widely applicable optimization method Score-DPO, a refined version of DPO designed to address these challenges. Our experiments demonstrate the superiority of Score-DPO in optimizing the LLM workflow generator, suggesting its suitability for similar settings. We elaborate on the two improvements of Score-DPO as follows.

\paragraph{Enhanced Sampling Distribution} The slow convergence and suboptimal performance observed when applying DPO in our setting can be attributed to inaccuracies in the collected preference data, caused by the unavoidable variance and error in evaluation scores. To address this, we propose up-weighting sample pairs \((w, l)\) with larger score differences \(s_w - s_l\). Specifically, we introduce a function \(d(x, y): [0, 1]^2 \rightarrow [0, 1]\) that is strictly monotonically increasing with respect to \(x - y\). 
We then up-weight the sampling probability of data pairs with larger score differences by increasing their likelihood according to \(P^{\star}(w, l) \propto d(s_w, s_l) P(w, l)\), where \(P(w, l)\) represents the uniform random sampling distribution over the preference dataset \(D_{pre}\). This adjustment ensures that pairs with greater score differences are prioritized during sampling, enhancing the effectiveness of the optimization process.

\paragraph{Incorporate Evaluation Scores into the Ranking Objective} There have been some alternative formulations for the Bradley-Terry (BT) \citep{bradley1952rank} ranking objective $\sigma (r_w - r_l)$ that are more effective than DPO \citep{meng2024simpo, azar2024general, park2024disentangling}, where $r_w := \beta \log(\pi_{\theta}(y_w | x)/ \pi_{ref}(y_w | x) )$ and $r_l := \beta \log(\pi_{\theta}(y_l | x)/ \pi_{ref}(y_l | x) )$. In our setting, we incorporate the evaluation score to guide the implicit reward. Specifically, we define the score-based BT ranking objective as $\sigma (r_w^{\star} - r_l^{\star})$, where \(r_w^{\star} := f(s_w) r_w\), \(r_l^{\star} := (1 - f(s_l)) r_l\), and \(f(x): [0, 1] \rightarrow [0, 1]\) is a strictly monotonically increasing function. 
Empirically, this approach ensures that data points with more deterministic evaluation scores have a greater influence on the loss function. Finally, we have the loss function of Score-DPO as 
\begin{align*}
\mathcal{L}_{\text{Score-DPO}} = - \mathbb{E}_{(w, l) \sim P^{\star}} [ \log \sigma (r^{\star}_{w} - r^{\star}_{l}) ].
\end{align*}

\subsection{Analysis of Score-DPO}

While DPO is known to struggle with effectively learning preference rankings \citep{chen2024preference}, the following theorem will demonstrate that this score-guided approach aligns the influence of each sample on the optimization objective with the magnitude of its evaluation scores.

To formalize our analysis, we introduce notation to quantify the influence of each specific sample on the optimization objective.

\begin{definition}[\textit{per-sample influence}]
For a given sample \( z \), the influence of \( z \) on the objective function, referred to as the \textit{per-sample influence}, is defined as:
\begin{align*}
I(z) = \frac{\partial}{\partial r_z} \mathbb{E}_{(w, l) \sim P^{\star}} \left[ \log \sigma \big(r^{\star}_{w} - r^{\star}_{l}\big) \cdot \mathds{1}_{z \in \{w, l\}} \right].
\end{align*}
\end{definition}

The per-sample influence \( I(z) \), which is the gradient contributed by sample \( z \), represents the quantitative impact of \( z \) on the optimization objective. When \(I(z) > 0\), the optimization process increases the logits of \(z\), making it more likely to be preferred. When \(I(z) < 0\), it decreases the logits of \(z\), making it less likely to be preferred. The following Theorem~\ref{deriscoreDPO} demonstrates the effect of score-guidance on \(I(z)\).

\begin{theorem}
\label{deriscoreDPO}
Let function \(d(x, y): [0, 1]^2 \rightarrow [0, 1]\) be strictly monotonically increasing with respect to \(x - y\), and function \(f(x): [0, 1] \rightarrow [0, 1]\) be strictly monotonically increasing in \(x\). The per-sample influence for a sample \(z\) is given by:
\setlength{\abovedisplayskip}{4pt}
\setlength{\belowdisplayskip}{4pt}
\begin{align*}
I(z) =  \mathbb{E}_{(w, l) \sim P} [ d(s_w, s_l) \sigma (r^{\star}_{l} - r^{\star}_{w}) \big( f(s_w) \mathds{1}_{w = z}  - (1 - f(s_l)) \mathds{1}_{l = z} \big) ],
\end{align*}
which is strictly monotonically increasing with the score \(s_z\) when $ -(1 - f(s_z))^{-1}\le r_z \le f^{-1}(s_z) $ holds.
\end{theorem}

Therefore, Score-DPO can incorporate score information into self-sampling preference optimization, enabling the optimization process to account for quantitative information, instead of only using the bare preference pairs information, and can reduce the error and variance caused by inaccuracies in the score.
Note that the condition stated in Theorem~\ref{deriscoreDPO} is not restrictive, as \( |r_z| \leq 1 \) provides a sufficient condition for its validity. Furthermore, our experimental results (in Appendix~\ref{appevfcit}) indicate that \( |r_z| \leq 1 \) holds with an approximate probability of \( 91.1\% \) during the optimization process prior to convergence.

\section{Experiments}

\subsection{Experimental Setup}

\paragraph{Datasets}
We focus on six public datasets, covering a range of tasks, including math problems, question-answering problems, and coding problems. Specifically, we utilize the full datasets for HumanEval \citep{chen2021evaluating} and MBPP \citep{austin2021program}. Following the approach of Aflow \citep{zhang2024aflow}, for GSM8K \citep{cobbe2021training}, we use the 1,319 data points in the test set. 
For the MATH dataset, to emphasize advanced and challenging problems, we select problems with a difficulty level of 5 from the following problem types: Combinatorics and Probability, Number Theory, Pre-algebra, and Pre-calculus, as done by \citet{hong2024data}. For DROP \citep{dua2019drop} and HotpotQA \citep{yang2018hotpotqa}, we follow the methodology outlined in \citet{hu2024automated}, \citet{shinn2024reflexion}, and \citet{zhang2024aflow}, randomly selecting 1,000 samples from each dataset. We split the data into validaton and test set using a 1:4 ratio.

\paragraph{Baselines}
The manually designed static workflow baselines include: direct LLM invocation, Chain of Thought \citep{wei2022chain}, Self-Consistency CoT (generate 5 responses to ensemble) \citep{wang2022self}, MedPrompt (3 responses and 5 votes) \citep{nori2023can}, MultiPersona Debate \citep{wang2023unleashing}, and Self-Refine (2 rounds) \citep{madaan2024self}. 
We also compare with code-representational automated workflow optimization methods: ADAS \citep{hu2024automated} and Aflow \citep{zhang2024aflow}, where we use GPT-4o-mini as their optimization model. We set the iteration rounds of Aflow to 20, as specified by \citet{zhang2024aflow}.

\paragraph{Models}
By default, we use Llama-3.1-8B-Instruct as the base model for our generator (inference performed using vLLM \citep{kwon2023efficient}), and GPT-4o-mini as the executor (inference via API, with a temperature of 0). 
In the ablation studies, we use Qwen2.5-7B-Instruct \citep{yang2024qwen2} as the generator and employ GPT-4o and DeepSeek series models \citep{deepseekai2024deepseekv3technicalreport} as the executors. All experiments used 2 A6000 GPUs using LoRA \citep{hu2021lora}.

\paragraph{Metrics and Evaluation Scores}
We report the solve rates (evaluated 3 times and averaged) in our final results. We use GPT-4o-mini as the judge model for MATH, DROP, and HotpotQA to avoid format inconsistency issues\footnote{When format inconsistencies arise, we use a judge model to resolve them (e.g., $0.1$ should equal to $10\%$).}. 
In each iteration of our optimization process (total 3 iterations), we generate $k = 8$ workflows for each problem and obtain their evaluation scores, where we do not use the judge model to reduce cost and computational overhead. Specifically, we use the F1 score as the evaluation metric for DROP and HotpotQA, and solve rates for the remaining datasets (evaluated 3 times and averaged). To apply Score-DPO, we set \( f(x) = x \) and \( d(x, y) = (x - y)^3 \) as the default choices. An ablation study on the selected functions is provided in Appendix~\ref{appalbd}.

\begin{table*}[t]
\scriptsize
\renewcommand\tabcolsep{3.2pt}
\renewcommand\arraystretch{1.2}
\small
\setlength{\abovecaptionskip}{0.1cm}
\setlength{\belowcaptionskip}{-0.2cm}
\centering
\caption{Comparison of performance between manually designed workflow methods and automated optimization workflow methods. All methods are executed using GPT-4o-mini, with each tested three times, and the average results reported.
}
\label{tab:mainresult}
\begin{tabular}{l|cc|cc|cc|c}
\specialrule{1.2pt}{0pt}{0pt} 
\multirow{2}{*}{\textbf{Method}} & \multicolumn{2}{c|}{\textbf{Question Answering}} & \multicolumn{2}{c|}{\textbf{Coding}} & \multicolumn{2}{c|}{\textbf{Math Reasoning}} & \multirow{2}{*}{\textbf{Average}} \\ 
& \textbf{HotpotQA} & \textbf{DROP} & \textbf{HumanEval} & \textbf{MBPP} & \textbf{GSM8K} & \textbf{MATH} & \\
\hline
\hline
IO 
& 73.6           & 81.6           & 90.1            & 69.5           & 89.1           & 52.2           & 76.0 \\

CoT \citep{wei2022chain} 
& 73.4           & 83.2           & 91.6            & 70.4           & 88.3           & 53.4           & 76.7 \\

CoT SC \citep{wang2022self} 
& 74.0           & 83.2           & 92.9            & 71.3           & 88.6           & 53.8           & 77.3 \\

MedPrompt \citep{nori2023can} 
& 74.4           & 83.0           & 92.1            & 69.2           & 88.1           & 53.7           & 76.8 \\

MultiPersona \citep{wang2023unleashing} 
& 73.1           & 81.3           & 92.9            & 70.4           & 89.8           & 51.9           & 76.5 \\

Self Refine \citep{madaan2024self} 
& 73.6           & 82.5           & 91.1            & 70.0           & 87.5           & 50.0           & 75.8 \\

ADAS \citep{hu2024automated} 
& 78.5           & 81.3           & 88.8            & 68.7           & 90.5           & 51.7           & 76.6 \\

Aflow \citep{zhang2024aflow} 
& 77.9           & 83.5           & 92.9            & 82.9           & 90.8           & 55.8           & 80.6 \\

\textbf{ScoreFlow (Ours)} 
& \textbf{86.0}  & \textbf{86.2}  & \textbf{95.9}   & \textbf{84.7}  & \textbf{94.6}  & \textbf{64.4}  & \textbf{85.3} \\
\specialrule{1.2pt}{0pt}{0pt} 
\end{tabular}
\vspace{-1mm}
\end{table*}

\begin{table*}[!t]
\scriptsize
\renewcommand\tabcolsep{3.2pt}
\renewcommand\arraystretch{1.2}
\small
\setlength{\abovecaptionskip}{0.1cm}
\setlength{\belowcaptionskip}{-0.2cm}
\centering
\caption{Comparison of different optimization methods within our ScoreFlow framework: We retain our pipeline, ScoreFlow, and replace the finetuning method to serve as baselines. Each method was tested three times, and we report the average solve rates on both validation and test set. The value on the left represents the performance on the validation set, while the value on the right represents the performance on the test set.
}
\label{tabfinetunemethod}
\scalebox{1.0}{%
\begin{tabular}{l|cc|cc|cc|cc|cc|cc}
\specialrule{1.2pt}{0pt}{0pt} 
\textbf{Method} 
& \multicolumn{2}{c|}{\textbf{HotpotQA}} 
& \multicolumn{2}{c|}{\textbf{DROP}} 
& \multicolumn{2}{c|}{\textbf{HumanEval}} 
& \multicolumn{2}{c|}{\textbf{MBPP}} 
& \multicolumn{2}{c|}{\textbf{GSM8K}} 
& \multicolumn{2}{c}{\textbf{MATH}} \\
\hline
\hline
SFT 
& 88.1 & 84.0 
& 85.5 & 82.3 
& 85.9 & 93.4 
& 83.5 & 82.0 
& 88.5 & 89.8 
& 49.6 & 54.8 \\

PPO    
& 87.9 & 84.2 
& 86.0 & 83.8 
& 84.8 & 92.7 
& 83.7 & 82.9 
& 87.7 & 89.2 
& 50.0 & 55.2 \\

DPO    
& 88.3 & 84.1 
& 85.3 & 84.2 
& 86.9 & 95.9 
& 84.1 & 82.9 
& 90.2 & 91.7 
& 53.6 & 60.4 \\

\textbf{Score-DPO (Ours)}
& \textbf{89.2} & \textbf{86.0} 
& \textbf{88.5} & \textbf{86.2} 
& \textbf{87.9} & \textbf{95.9} 
& \textbf{86.0} & \textbf{84.7} 
& \textbf{93.7} & \textbf{94.6} 
& \textbf{56.5} & \textbf{64.4} \\
\specialrule{1.2pt}{0pt}{0pt} 
\end{tabular}%
}
\vspace{-3mm}
\end{table*}

\subsection{Results and Analysis}
\label{secexp}

\paragraph{Main Results}
The main results are presented in Table~\ref{tab:mainresult}. Our proposed method, ScoreFlow, consistently outperforms all manually designed workflow methods as well as automated workflow optimization methods included in the baselines across all benchmarks. Notably, our method achieves an average solve rate of $85.3\%$, surpassing the baseline methods by a margin of $8.2\%$. The two automated workflow optimization methods, despite employing GPT-4o-mini as the workflow generator, consistently underperform compared to our approach, which utilizes a significantly smaller 8B model as the generator, across all tasks. These results highlight the robustness and effectiveness of ScoreFlow in optimizing workflows and achieving improved performance across diverse tasks.

\begin{table*}[tp]
\scriptsize
\renewcommand\tabcolsep{3.2pt}
\renewcommand\arraystretch{1.2}
\small
\setlength{\abovecaptionskip}{0.1cm}
\setlength{\belowcaptionskip}{-0.2cm}
\centering
\caption{Comparison of Performance Across Different Models and Methods on HumanEval task.
We conduct ablation studies on both the generator and the executor. For the generator ablation, we use Llama-3.1-8B-Instruct (Ours) and Qwen2.5-7B-Instruct (Ours*). For the executor ablation, we employ GPT-4o-mini, GPT-4o, DeepSeek-V3, and DeepSeek-coder. Each method is evaluated three times, and we report the average results.}
\label{tab:ablationmodel}
\begin{tabular}{l|ccccccccc}
\specialrule{1.2pt}{0pt}{0pt} 
\textbf{Executor} & \textbf{Ours} & \textbf{Ours*} & \textbf{Aflow} & \textbf{IO} & \textbf{CoT} & \textbf{CoT SC} & \textbf{MP} & \textbf{MPD} & \textbf{SR} \\
\hline
GPT-4o-mini 
& 95.7 & 95.1 & 92.9 & 90.1 & 91.6 & 92.9 & 92.1 & 92.9 & 91.1 \\

GPT-4o 
& \textbf{97.7} & 97.4 & 94.7 & 93.1 & 93.4 & 93.9 & 95.9 & 96.2 & 92.6 \\

DeepSeek-V3 
& 97.2 & 96.9 & 94.7 & 90.8 & 90.1 & 93.4 & 93.9 & 92.9 & 93.9 \\

DeepSeek-coder 
& \textbf{97.7} & 96.7 & 93.4 & 91.3 & 92.4 & 94.7 & 95.2 & 94.4 & 94.4 \\
\specialrule{1.2pt}{0pt}{0pt} 
\end{tabular}
\vspace{0mm}
\end{table*}

\paragraph{Improvements of Proposed Score-DPO}
To demonstrate the utility of our preference optimization method, Score-DPO, we compare our results with additional designed baselines that replace our finetuning method with alternative approaches (while retaining our overall pipeline, ScoreFlow): supervised finetuning (SFT), proximal policy optimization (PPO), and direct preference optimization (DPO). 
For SFT, we select the preferred responses to fine-tune the generator, where the preferred responses are sampled using score-sampling (part of our proposed method). For PPO, we follow \citet{huang2024n+}: Firstly train a reward model (share the same base model of generator) using the collected preference data and then optimize the generator based on the reward model. For raw DPO, we directly use the collected preference data and the original Bradley-Terry model, performing optimization for 3 iterations. 
We report the solve rates on both the validation set and the test set.
Table~\ref{tabfinetunemethod} demonstrates the effectiveness of our proposed method, Score-DPO. The SFT method only provides the generator with information about preferred responses, neglecting the rejected responses. Using PPO to optimize over long token sequences in our setting can dilute gradient signals, making it more difficult for the model to discern which parts of the sequence contribute most to the reward. This leads to instability and degraded performance \citep{schulman2017proximal, chen2021evaluating}. Score-DPO incorporates specific evaluation ranking information into DPO, achieving both efficiency and the best performance among the baselines.

\paragraph{Gradient Loss Optimization and Adaptivity Enhance Scalability}
We demonstrate how the loss-gradient optimization method, combined with an adaptive framework, enhances scalability by maintaining high performance when applied to more diverse and larger problem datasets. We conduct a comparison with the second-best performing baseline method, AFlow. In this experiment, we integrate math, coding, and question-answering tasks by selecting datasets where Aflow shows the smallest performance difference with ScoreFlow: GSM8K (math), MBPP (coding), and DROP (question answering). These datasets are then combined for optimization and evaluation. Figure~\ref{differencefig} demonstrates that ScoreFlow achieves a more pronounced performance advantage over AFlow on the more diverse combined dataset. AFlow employs a standard discrete optimization method to optimize a single workflow, where a few failed cases are randomly selected and fed into the optimizer LLM to refine the workflow in each iteration. The discrete optimization approach and lack of adaptability could limit scalability. In contrast, our adaptive generation framework, coupled with the loss-gradient optimization method, effectively addresses these challenges.

\paragraph{Case Study on Adaptivity} In ScoreFlow, the task information is provided to the generator to facilitate the creation of adaptive workflows. Specifically, the generator has the flexibility to select appropriate operators and adapt the complexity of the workflow structure based on the characteristics of the given problem.
In Figure~\ref{figspecificeample}, two distinct GSM8K problems are each assigned unique workflows, with the correct answer achievable only when the corresponding workflow is utilized.  For the more complex and computation-intensive Problem B, constructing a sophisticated workflow with program and review operators helps mitigate calculation errors, while calculation errors occur when using a simple workflow. Conversely, for the simple, concise, and calculation-light Problem A, employing an overly complex workflow can result in overthinking and inefficiency. This underscores the critical role of adaptivity in workflow generation.

\begin{figure}[H]
    \centering
    \includegraphics[width=0.4\textwidth]{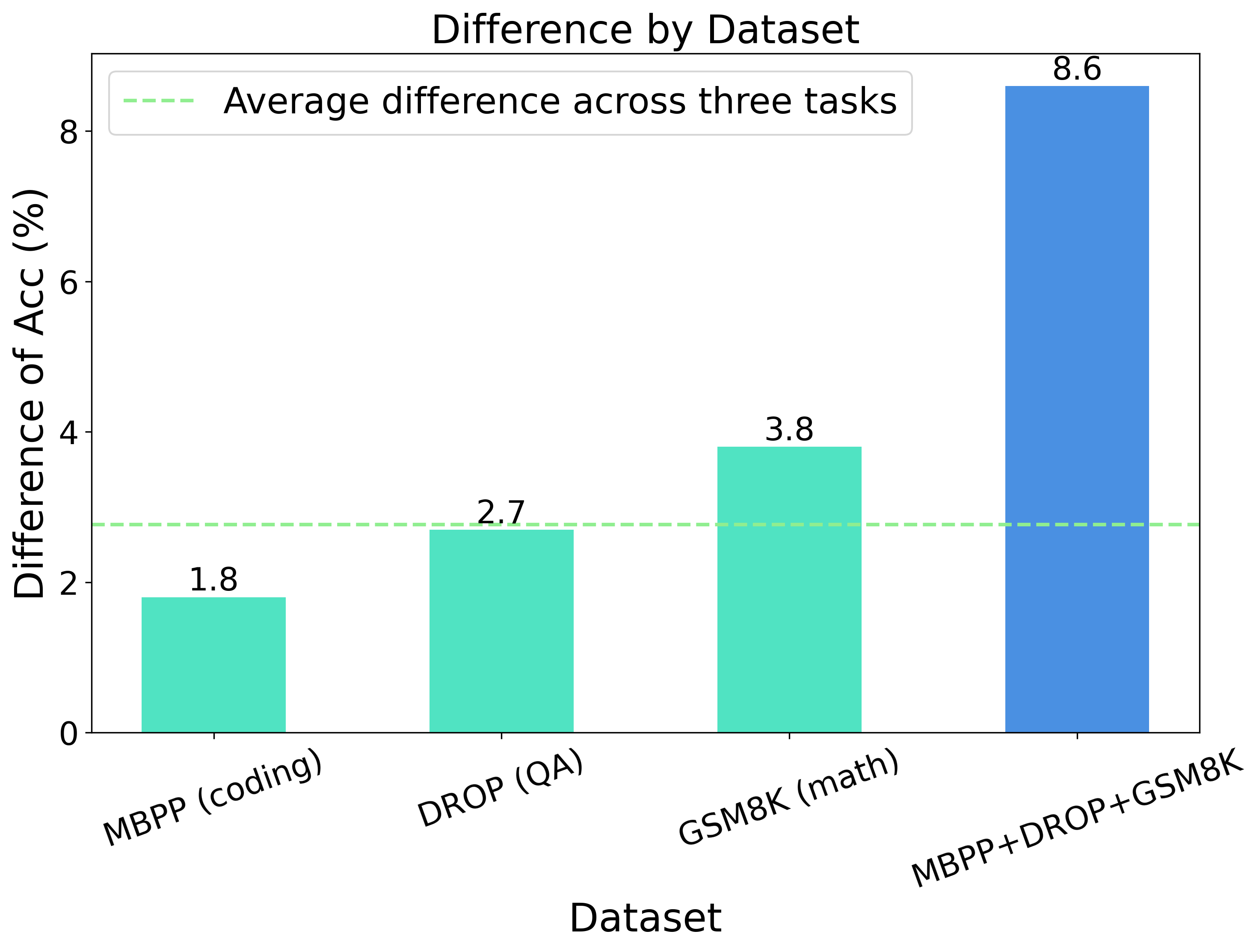}
    \caption{Performance comparison between ScoreFlow and Aflow across various datasets. The y-axis represents the difference in accuracy ($\%$), calculated as the win rate of ScoreFlow minus the win rate of AFlow on test set. The executor for both methods are GPT-4o-mini. The optimizer LLM (generator) for Aflow is GPT-4o-mini, while the generator for Scoreflow is Llama-3.1-8B-Instruct. Specifically, ScoreFlow achieves a $88.1\%$ performance on the combined task.}
    \label{differencefig}
    \vspace{-3mm}
\end{figure}

\paragraph{Robust for Different LLM Architectures} We conduct ablation studies on both the generator and the executor. For the generator ablation, we use Llama-3.1-8B-Instruct and Qwen2.5-7B-Instruct. For the executor ablation, we employ GPT-4o-mini, GPT-4o, DeepSeek-V3, and DeepSeek-coder. Specifically, in the GPT-4o setting, we utilize GPT-4o-mini during the optimization process and switch to GPT-4o at test time, as GPT-4o is prohibitively expensive for optimization in both ScoreFlow and Aflow.
From the results in Table~\ref{tab:ablationmodel}, we first demonstrate the robustness of our method by showing that it consistently outperforms baseline methods across various combinations of generators and executors.
Second, we observe that ScoreFlow, when utilizing smaller models such as GPT-4o-mini and DeepSeek-V3, outperforms the Chain-of-Thought (CoT) outputs of the larger GPT-4o model.
The best performance is achieved when GPT-4o or DeepSeek-coder is used as the generator. Notably, although DeepSeek-V3 exhibits a performance gap compared to GPT-4o and DeepSeek-coder when evaluated as a standalone agent, workflow optimization enables DeepSeek-V3 to achieve performance comparable to that of GPT-4o and DeepSeek-coder. This result highlights the effectiveness of our proposed method.

\begin{figure*}[t]
    \centering
    \begin{subfigure}[t]{0.45\textwidth}
        \centering
        \includegraphics[height=5.3cm]{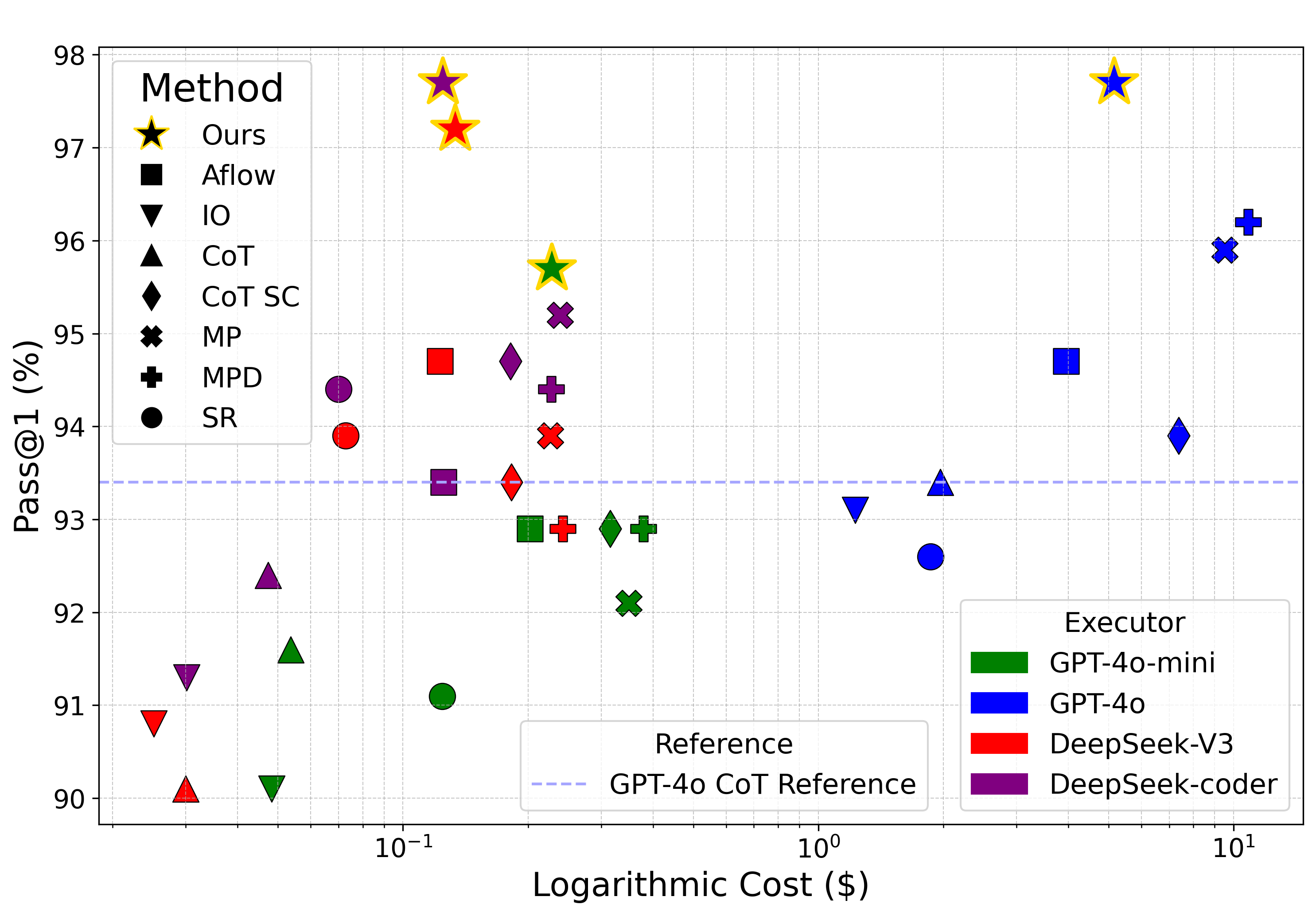}
        \caption{Cost during inference on testing set.}
        \label{figcostinference}
    \end{subfigure}
    \hspace{0.02\textwidth} %
    \begin{subfigure}[t]{0.45\textwidth}
        \centering
        \includegraphics[height=5.3cm]{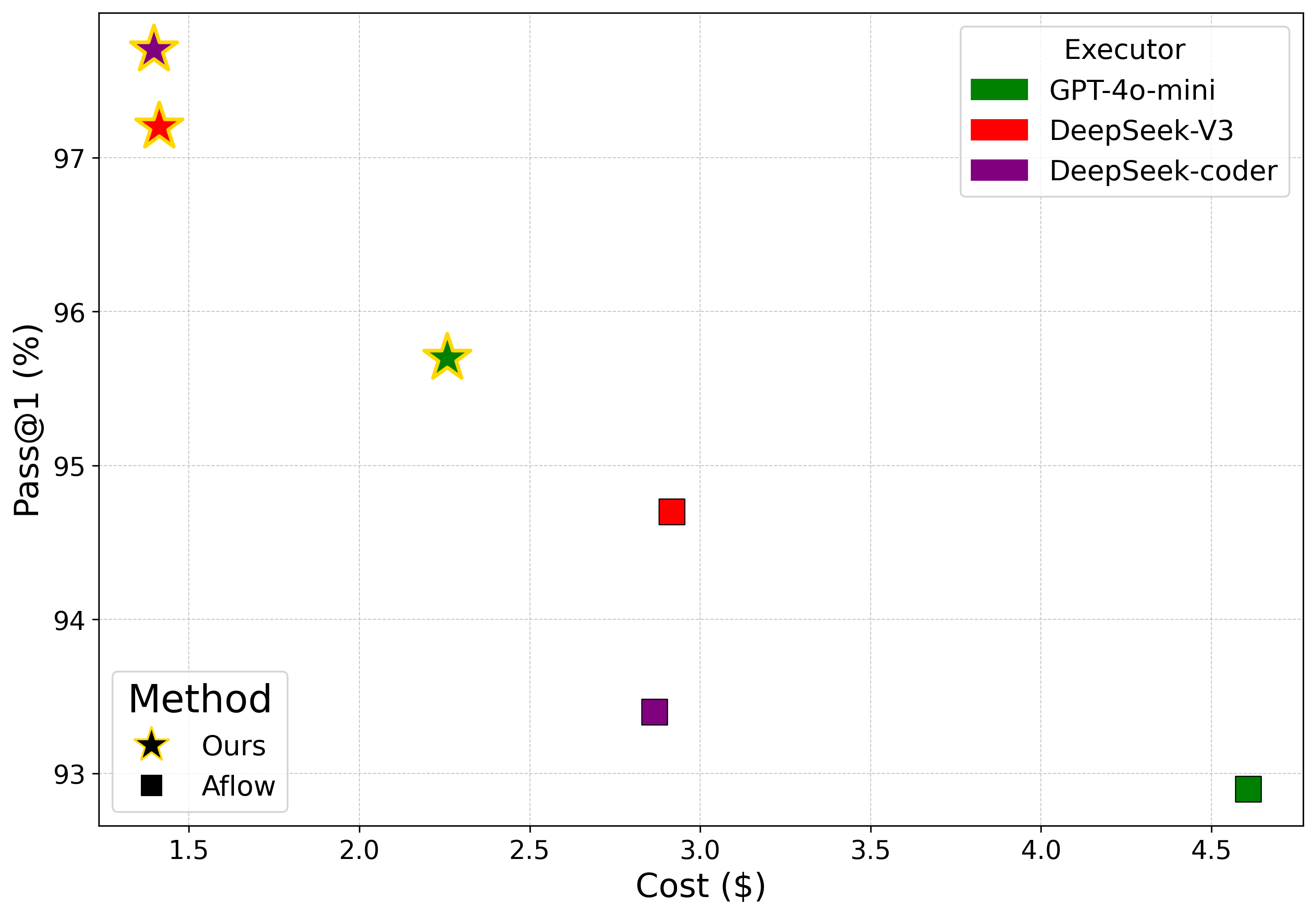}
        \caption{Cost during optimization.}
        \label{figcostoptimization}
    \end{subfigure}
    \caption{API Cost in Inference and Optimization processes.
We analyze the API cost during both the inference and optimization processes, comparing different methods across various executors for the HumanEval task. The left figure illustrates the cost during inference on the testing set in relation to Pass@1 performance. The right figure highlights the total cost of optimization for ScoreFlow and AFlow. The generator for our method here is Llama-3.1-8B-Instruct.}
\label{figcost}
\vspace{-3mm}
\end{figure*}

\paragraph{Cost Efficiency} 
Using an open-source LLM as the base model and leveraging fast convergence in optimization minimizes the expense of our method. We firstly analyze the API costs during the inference stage for different methods, across 4 different versions of executors, focusing on the HumanEval task. Results in Figure~\ref{figcostinference} demonstrate that ScoreFlow enables weaker models to achieve better cost-effectiveness than stronger models, balancing performance and resource usage optimally. For example, ScoreFlow utilizes smaller models such as GPT-4o-mini, DeepSeek-V3, and DeepSeek-coder to achieve significantly better performance than GPT-4o's CoT approach, while maintaining much lower costs.
Compared to the inference stage, the optimization process in automated workflow optimization methods is more computationally expensive, as it requires evaluation feedback from the executor at each iteration. Therefore we also compare the expense during optimization process with Aflow. From Figure~\ref{figcostoptimization}, we demonstrate that ScoreFlow consistently costs less than Aflow in the optimization process while performing better, which highlights the cost-efficiency of our method.

\paragraph{Iterative Process Analysis} Figure~\ref{iterprocessfig} illustrates the changes in solve rate during the iterative process. 
The consistent increase in test solve rate, followed by its eventual convergence, demonstrates the effectiveness of the iterative approach. We observed that rapid convergence can be achieved as early as the second iteration in our study.

\begin{figure}[ht]
    \centering
    \includegraphics[width=0.7\textwidth]{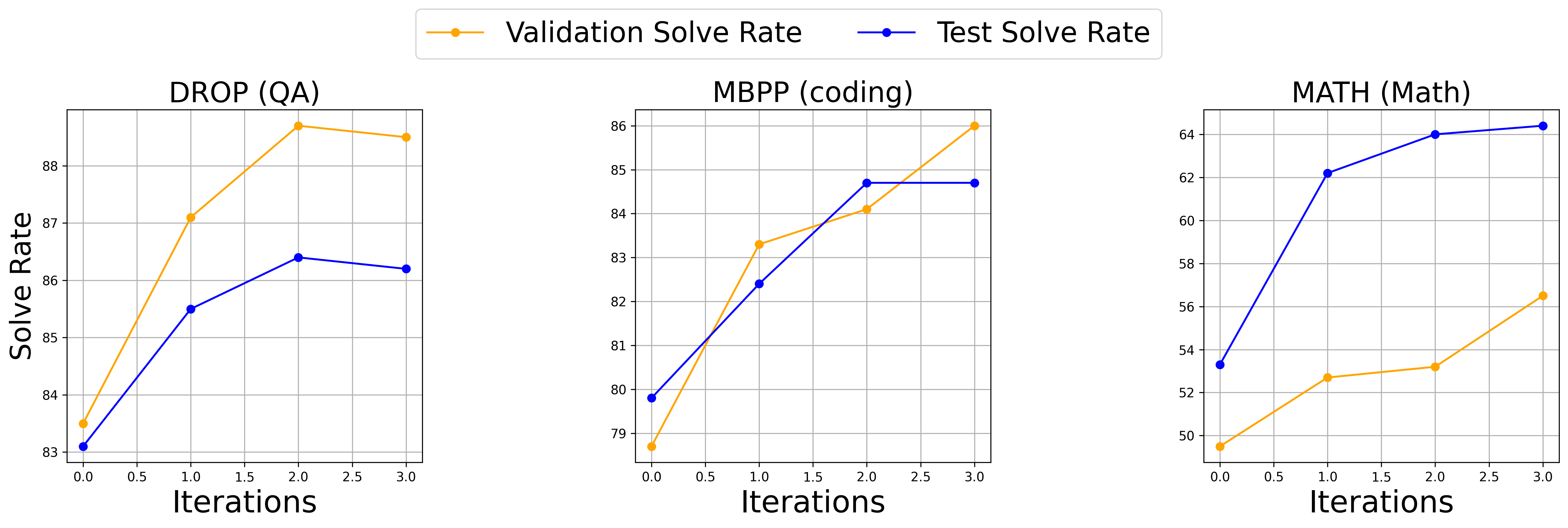}
    \caption{Solve rate during iteration process.}
    \label{iterprocessfig}
\vspace{-5mm}
\end{figure}

\section{Conclusion}

In this work, we propose ScoreFlow, an automated, high-performance, and adaptive framework for optimizing multi-agent workflows. The framework leverages the generalizable Score-DPO to achieve robust and efficient optimization. By replacing traditional discrete optimization algorithms with loss-gradient-based optimization, we enhance the framework's flexibility and scalability. Score-DPO, as an effective preference optimization method, reduces inaccuracies and variances in collected data pairs, thereby improving overall performance by incorporating evaluation scores directly into the optimization process.

By evaluating six benchmarks spanning question answering, coding, and mathematical reasoning tasks, ScoreFlow achieves an average improvement of 8.2\% over baseline methods. Additionally, Score-DPO consistently outperforms widely used preference optimization methods. Comprehensive ablation studies across various models highlight the robustness and cost-efficiency of our approach. Notably, our method enables smaller models to outperform larger models while incurring lower API costs.

\bibliography{main}

\begin{thebibliography}{49}
\providecommand{\natexlab}[1]{#1}
\providecommand{\url}[1]{\texttt{#1}}
\expandafter\ifx\csname urlstyle\endcsname\relax
  \providecommand{\doi}[1]{doi: #1}\else
  \providecommand{\doi}{doi: \begingroup \urlstyle{rm}\Url}\fi

\bibitem[Achiam et~al.(2023)Achiam, Adler, Agarwal, Ahmad, Akkaya, Aleman,
  Almeida, Altenschmidt, Altman, Anadkat, et~al.]{achiam2023gpt}
Josh Achiam, Steven Adler, Sandhini Agarwal, Lama Ahmad, Ilge Akkaya,
  Florencia~Leoni Aleman, Diogo Almeida, Janko Altenschmidt, Sam Altman,
  Shyamal Anadkat, et~al.
\newblock Gpt-4 technical report.
\newblock \emph{arXiv preprint arXiv:2303.08774}, 2023.

\bibitem[Anil et~al.(2023)Anil, Dai, Firat, Johnson, Lepikhin, Passos, Shakeri,
  Taropa, Bailey, Chen, et~al.]{anil2023palm}
Rohan Anil, Andrew~M Dai, Orhan Firat, Melvin Johnson, Dmitry Lepikhin,
  Alexandre Passos, Siamak Shakeri, Emanuel Taropa, Paige Bailey, Zhifeng Chen,
  et~al.
\newblock Palm 2 technical report.
\newblock \emph{arXiv preprint arXiv:2305.10403}, 2023.

\bibitem[Austin et~al.(2021)Austin, Odena, Nye, Bosma, Michalewski, Dohan,
  Jiang, Cai, Terry, Le, et~al.]{austin2021program}
Jacob Austin, Augustus Odena, Maxwell Nye, Maarten Bosma, Henryk Michalewski,
  David Dohan, Ellen Jiang, Carrie Cai, Michael Terry, Quoc Le, et~al.
\newblock Program synthesis with large language models.
\newblock \emph{arXiv preprint arXiv:2108.07732}, 2021.

\bibitem[Azar et~al.(2024)Azar, Guo, Piot, Munos, Rowland, Valko, and
  Calandriello]{azar2024general}
Mohammad~Gheshlaghi Azar, Zhaohan~Daniel Guo, Bilal Piot, Remi Munos, Mark
  Rowland, Michal Valko, and Daniele Calandriello.
\newblock A general theoretical paradigm to understand learning from human
  preferences.
\newblock In \emph{International Conference on Artificial Intelligence and
  Statistics}, pages 4447--4455. PMLR, 2024.

\bibitem[Bradley and Terry(1952)]{bradley1952rank}
Ralph~Allan Bradley and Milton~E Terry.
\newblock Rank analysis of incomplete block designs: I. the method of paired
  comparisons.
\newblock \emph{Biometrika}, 39\penalty0 (3/4):\penalty0 324--345, 1952.

\bibitem[Chen et~al.(2024)Chen, Malladi, Zhang, Chen, Zhang, Ranganath, and
  Cho]{chen2024preference}
Angelica Chen, Sadhika Malladi, Lily~H Zhang, Xinyi Chen, Qiuyi Zhang, Rajesh
  Ranganath, and Kyunghyun Cho.
\newblock Preference learning algorithms do not learn preference rankings.
\newblock \emph{arXiv preprint arXiv:2405.19534}, 2024.

\bibitem[Chen et~al.(2023)Chen, Dong, Shu, Zhang, Sesay, Karlsson, Fu, and
  Shi]{chen2023autoagents}
Guangyao Chen, Siwei Dong, Yu~Shu, Ge~Zhang, Jaward Sesay, B{\"o}rje~F
  Karlsson, Jie Fu, and Yemin Shi.
\newblock Autoagents: A framework for automatic agent generation.
\newblock \emph{arXiv preprint arXiv:2309.17288}, 2023.

\bibitem[Chen et~al.(2021)Chen, Tworek, Jun, Yuan, Pinto, Kaplan, Edwards,
  Burda, Joseph, Brockman, et~al.]{chen2021evaluating}
Mark Chen, Jerry Tworek, Heewoo Jun, Qiming Yuan, Henrique Ponde De~Oliveira
  Pinto, Jared Kaplan, Harri Edwards, Yuri Burda, Nicholas Joseph, Greg
  Brockman, et~al.
\newblock Evaluating large language models trained on code.
\newblock \emph{arXiv preprint arXiv:2107.03374}, 2021.

\bibitem[Cobbe et~al.(2021)Cobbe, Kosaraju, Bavarian, Chen, Jun, Kaiser,
  Plappert, Tworek, Hilton, Nakano, et~al.]{cobbe2021training}
Karl Cobbe, Vineet Kosaraju, Mohammad Bavarian, Mark Chen, Heewoo Jun, Lukasz
  Kaiser, Matthias Plappert, Jerry Tworek, Jacob Hilton, Reiichiro Nakano,
  et~al.
\newblock Training verifiers to solve math word problems.
\newblock \emph{arXiv preprint arXiv:2110.14168}, 2021.

\bibitem[Dua et~al.(2019)Dua, Wang, Dasigi, Stanovsky, Singh, and
  Gardner]{dua2019drop}
Dheeru Dua, Yizhong Wang, Pradeep Dasigi, Gabriel Stanovsky, Sameer Singh, and
  Matt Gardner.
\newblock {DROP}: A reading comprehension benchmark requiring discrete
  reasoning over paragraphs.
\newblock In \emph{Proc. of NAACL}, 2019.

\bibitem[Fernando et~al.(2023)Fernando, Banarse, Michalewski, Osindero, and
  Rockt{\"a}schel]{fernando2023promptbreeder}
Chrisantha Fernando, Dylan Banarse, Henryk Michalewski, Simon Osindero, and Tim
  Rockt{\"a}schel.
\newblock Promptbreeder: Self-referential self-improvement via prompt
  evolution.
\newblock \emph{arXiv preprint arXiv:2309.16797}, 2023.

\bibitem[Hong et~al.(2024)Hong, Lin, Liu, Liu, Wu, Zhang, Wei, Li, Chen, Zhang,
  et~al.]{hong2024data}
Sirui Hong, Yizhang Lin, Bang Liu, Bangbang Liu, Binhao Wu, Ceyao Zhang,
  Chenxing Wei, Danyang Li, Jiaqi Chen, Jiayi Zhang, et~al.
\newblock Data interpreter: An llm agent for data science.
\newblock \emph{arXiv preprint arXiv:2402.18679}, 2024.

\bibitem[Hu et~al.(2022)Hu, Shen, Wallis, Allen-Zhu, Li, Wang, and
  Chen]{hu2021lora}
Edward~J. Hu, Yelong Shen, Phillip Wallis, Zeyuan Allen-Zhu, Yuanzhi Li, Shean
  Wang, and Weizhu Chen.
\newblock Lora: Low-rank adaptation of large language models.
\newblock In \emph{International Conference on Learning Representations}, 2022.

\bibitem[Hu et~al.(2024)Hu, Lu, and Clune]{hu2024automated}
Shengran Hu, Cong Lu, and Jeff Clune.
\newblock Automated design of agentic systems.
\newblock In \emph{NeurIPS 2024 Workshop on Open-World Agents}, 2024.

\bibitem[Huang et~al.(2024)Huang, Noukhovitch, Hosseini, Rasul, Wang, and
  Tunstall]{huang2024n+}
Shengyi Huang, Michael Noukhovitch, Arian Hosseini, Kashif Rasul, Weixun Wang,
  and Lewis Tunstall.
\newblock The n+ implementation details of rlhf with ppo: A case study on tl;
  dr summarization.
\newblock \emph{arXiv preprint arXiv:2403.17031}, 2024.

\bibitem[Jiang et~al.(2024)Jiang, Zhang, Weller, Weir, Van~Durme, and
  Khashabi]{jiang2024self}
Dongwei Jiang, Jingyu Zhang, Orion Weller, Nathaniel Weir, Benjamin Van~Durme,
  and Daniel Khashabi.
\newblock Self-[in] correct: Llms struggle with refining self-generated
  responses.
\newblock \emph{arXiv preprint arXiv:2404.04298}, 2024.

\bibitem[Khattab et~al.(2024)Khattab, Singhvi, Maheshwari, Zhang, Santhanam,
  Haq, Sharma, Joshi, Moazam, Miller, et~al.]{khattab2024dspy}
Omar Khattab, Arnav Singhvi, Paridhi Maheshwari, Zhiyuan Zhang, Keshav
  Santhanam, Saiful Haq, Ashutosh Sharma, Thomas~T Joshi, Hanna Moazam, Heather
  Miller, et~al.
\newblock Dspy: Compiling declarative language model calls into
  state-of-the-art pipelines.
\newblock In \emph{The Twelfth International Conference on Learning
  Representations}, 2024.

\bibitem[Kwon et~al.(2023)Kwon, Li, Zhuang, Sheng, Zheng, Yu, Gonzalez, Zhang,
  and Stoica]{kwon2023efficient}
Woosuk Kwon, Zhuohan Li, Siyuan Zhuang, Ying Sheng, Lianmin Zheng, Cody~Hao Yu,
  Joseph~E. Gonzalez, Hao Zhang, and Ion Stoica.
\newblock Efficient memory management for large language model serving with
  pagedattention.
\newblock In \emph{Proceedings of the ACM SIGOPS 29th Symposium on Operating
  Systems Principles}, 2023.

\bibitem[Li et~al.(2024)Li, Xu, Mei, Hua, Rama, Raheja, Wang, Zhu, and
  Zhang]{li2024autoflow}
Zelong Li, Shuyuan Xu, Kai Mei, Wenyue Hua, Balaji Rama, Om~Raheja, Hao Wang,
  He~Zhu, and Yongfeng Zhang.
\newblock Autoflow: Automated workflow generation for large language model
  agents.
\newblock \emph{arXiv preprint arXiv:2407.12821}, 2024.

\bibitem[Liu et~al.(2024{\natexlab{a}})Liu, Feng, Xue, Wang, Wu, Lu, Zhao,
  Deng, Zhang, Ruan, et~al.]{deepseekai2024deepseekv3technicalreport}
Aixin Liu, Bei Feng, Bing Xue, Bingxuan Wang, Bochao Wu, Chengda Lu, Chenggang
  Zhao, Chengqi Deng, Chenyu Zhang, Chong Ruan, et~al.
\newblock Deepseek-v3 technical report.
\newblock \emph{arXiv preprint arXiv:2412.19437}, 2024{\natexlab{a}}.

\bibitem[Liu et~al.(2024{\natexlab{b}})Liu, Zhang, Li, Liu, and
  Yang]{liu2024dynamic}
Zijun Liu, Yanzhe Zhang, Peng Li, Yang Liu, and Diyi Yang.
\newblock A dynamic llm-powered agent network for task-oriented agent
  collaboration.
\newblock In \emph{First Conference on Language Modeling}, 2024{\natexlab{b}}.

\bibitem[Madaan et~al.(2024)Madaan, Tandon, Gupta, Hallinan, Gao, Wiegreffe,
  Alon, Dziri, Prabhumoye, Yang, et~al.]{madaan2024self}
Aman Madaan, Niket Tandon, Prakhar Gupta, Skyler Hallinan, Luyu Gao, Sarah
  Wiegreffe, Uri Alon, Nouha Dziri, Shrimai Prabhumoye, Yiming Yang, et~al.
\newblock Self-refine: Iterative refinement with self-feedback.
\newblock \emph{Advances in Neural Information Processing Systems}, 36, 2024.

\bibitem[Meng et~al.(2024)Meng, Xia, and Chen]{meng2024simpo}
Yu~Meng, Mengzhou Xia, and Danqi Chen.
\newblock Simpo: Simple preference optimization with a reference-free reward.
\newblock \emph{arXiv preprint arXiv:2405.14734}, 2024.

\bibitem[Nori et~al.(2023)Nori, Lee, Zhang, Carignan, Edgar, Fusi, King,
  Larson, Li, Liu, et~al.]{nori2023can}
Harsha Nori, Yin~Tat Lee, Sheng Zhang, Dean Carignan, Richard Edgar, Nicolo
  Fusi, Nicholas King, Jonathan Larson, Yuanzhi Li, Weishung Liu, et~al.
\newblock Can generalist foundation models outcompete special-purpose tuning?
  case study in medicine.
\newblock \emph{arXiv preprint arXiv:2311.16452}, 2023.

\bibitem[Ouyang et~al.(2022)Ouyang, Wu, Jiang, Almeida, Wainwright, Mishkin,
  Zhang, Agarwal, Slama, Ray, et~al.]{ouyang2022training}
Long Ouyang, Jeffrey Wu, Xu~Jiang, Diogo Almeida, Carroll Wainwright, Pamela
  Mishkin, Chong Zhang, Sandhini Agarwal, Katarina Slama, Alex Ray, et~al.
\newblock Training language models to follow instructions with human feedback.
\newblock \emph{Advances in neural information processing systems},
  35:\penalty0 27730--27744, 2022.

\bibitem[Park et~al.(2024)Park, Rafailov, Ermon, and
  Finn]{park2024disentangling}
Ryan Park, Rafael Rafailov, Stefano Ermon, and Chelsea Finn.
\newblock Disentangling length from quality in direct preference optimization.
\newblock In \emph{Findings of the Association for Computational Linguistics
  (ACL 2024)}, 2024.

\bibitem[Rafailov et~al.(2024)Rafailov, Sharma, Mitchell, Manning, Ermon, and
  Finn]{rafailov2024direct}
Rafael Rafailov, Archit Sharma, Eric Mitchell, Christopher~D Manning, Stefano
  Ermon, and Chelsea Finn.
\newblock Direct preference optimization: Your language model is secretly a
  reward model.
\newblock \emph{Advances in Neural Information Processing Systems}, 36, 2024.

\bibitem[Ridnik et~al.(2024)Ridnik, Kredo, and Friedman]{ridnik2024code}
Tal Ridnik, Dedy Kredo, and Itamar Friedman.
\newblock Code generation with alphacodium: From prompt engineering to flow
  engineering.
\newblock \emph{arXiv preprint arXiv:2401.08500}, 2024.

\bibitem[Saad-Falcon et~al.(2024)Saad-Falcon, Lafuente, Natarajan, Maru,
  Todorov, Guha, Buchanan, Chen, Guha, R{\'e}, et~al.]{saad2024archon}
Jon Saad-Falcon, Adrian~Gamarra Lafuente, Shlok Natarajan, Nahum Maru, Hristo
  Todorov, Etash Guha, E~Kelly Buchanan, Mayee Chen, Neel Guha, Christopher
  R{\'e}, et~al.
\newblock Archon: An architecture search framework for inference-time
  techniques.
\newblock \emph{arXiv preprint arXiv:2409.15254}, 2024.

\bibitem[Schulman et~al.(2017)Schulman, Wolski, Dhariwal, Radford, and
  Klimov]{schulman2017proximal}
John Schulman, Filip Wolski, Prafulla Dhariwal, Alec Radford, and Oleg Klimov.
\newblock Proximal policy optimization algorithms.
\newblock \emph{arXiv preprint arXiv:1707.06347}, 2017.

\bibitem[Shinn et~al.(2024)Shinn, Cassano, Gopinath, Narasimhan, and
  Yao]{shinn2024reflexion}
Noah Shinn, Federico Cassano, Ashwin Gopinath, Karthik Narasimhan, and Shunyu
  Yao.
\newblock Reflexion: Language agents with verbal reinforcement learning.
\newblock \emph{Advances in Neural Information Processing Systems}, 36, 2024.

\bibitem[Song et~al.(2024)Song, Liu, Zhang, Zhang, Luo, Wang, Wu, and
  Wang]{song2024adaptive}
Linxin Song, Jiale Liu, Jieyu Zhang, Shaokun Zhang, Ao~Luo, Shijian Wang,
  Qingyun Wu, and Chi Wang.
\newblock Adaptive in-conversation team building for language model agents.
\newblock \emph{arXiv preprint arXiv:2405.19425}, 2024.

\bibitem[Touvron et~al.(2023)Touvron, Martin, Stone, Albert, Almahairi, Babaei,
  Bashlykov, Batra, Bhargava, Bhosale, et~al.]{touvron2023llama}
Hugo Touvron, Louis Martin, Kevin Stone, Peter Albert, Amjad Almahairi, Yasmine
  Babaei, Nikolay Bashlykov, Soumya Batra, Prajjwal Bhargava, Shruti Bhosale,
  et~al.
\newblock Llama 2: Open foundation and fine-tuned chat models.
\newblock \emph{arXiv preprint arXiv:2307.09288}, 2023.

\bibitem[Wang et~al.(2022)Wang, Wei, Schuurmans, Le, Chi, Narang, Chowdhery,
  and Zhou]{wang2022self}
Xuezhi Wang, Jason Wei, Dale Schuurmans, Quoc Le, Ed~Chi, Sharan Narang,
  Aakanksha Chowdhery, and Denny Zhou.
\newblock Self-consistency improves chain of thought reasoning in language
  models.
\newblock \emph{The Eleventh International Conference on Learning
  Representations}, 2022.

\bibitem[Wang et~al.(2024)Wang, Mao, Wu, Ge, Wei, and Ji]{wang2023unleashing}
Zhenhailong Wang, Shaoguang Mao, Wenshan Wu, Tao Ge, Furu Wei, and Heng Ji.
\newblock Unleashing the emergent cognitive synergy in large language models: A
  task-solving agent through multi-persona self-collaboration.
\newblock In \emph{Proceedings of the 2024 Conference of the North American
  Chapter of the Association for Computational Linguistics: Human Language
  Technologies (Volume 1: Long Papers)}, pages 257--279, 2024.

\bibitem[Wei et~al.(2022)Wei, Wang, Schuurmans, Bosma, Xia, Chi, Le, Zhou,
  et~al.]{wei2022chain}
Jason Wei, Xuezhi Wang, Dale Schuurmans, Maarten Bosma, Fei Xia, Ed~Chi, Quoc~V
  Le, Denny Zhou, et~al.
\newblock Chain-of-thought prompting elicits reasoning in large language
  models.
\newblock \emph{Advances in neural information processing systems},
  35:\penalty0 24824--24837, 2022.

\bibitem[Xu et~al.(2024{\natexlab{a}})Xu, Fu, Gao, Ye, Liu, Mei, Wang, Yu, and
  Wu]{xu2024dpo}
Shusheng Xu, Wei Fu, Jiaxuan Gao, Wenjie Ye, Weilin Liu, Zhiyu Mei, Guangju
  Wang, Chao Yu, and Yi~Wu.
\newblock Is dpo superior to ppo for llm alignment? a comprehensive study.
\newblock In \emph{Proceedings of the 41st International Conference on Machine
  Learning (ICML)}, 2024{\natexlab{a}}.

\bibitem[Xu et~al.(2024{\natexlab{b}})Xu, Hongjin, Xing, Mi, Liu, Shi, Hui,
  Zhou, Liu, Xie, et~al.]{xu2023lemur}
Yiheng Xu, SU~Hongjin, Chen Xing, Boyu Mi, Qian Liu, Weijia Shi, Binyuan Hui,
  Fan Zhou, Yitao Liu, Tianbao Xie, et~al.
\newblock Lemur: Harmonizing natural language and code for language agents.
\newblock In \emph{The Twelfth International Conference on Learning
  Representations}, 2024{\natexlab{b}}.

\bibitem[Yang et~al.(2024{\natexlab{a}})Yang, Yang, Zhang, Hui, Zheng, Yu, Li,
  Liu, Huang, Wei, et~al.]{yang2024qwen2}
An~Yang, Baosong Yang, Beichen Zhang, Binyuan Hui, Bo~Zheng, Bowen Yu,
  Chengyuan Li, Dayiheng Liu, Fei Huang, Haoran Wei, et~al.
\newblock Qwen2. 5 technical report.
\newblock \emph{arXiv preprint arXiv:2412.15115}, 2024{\natexlab{a}}.

\bibitem[Yang et~al.(2023)Yang, Wang, Lu, Liu, Le, Zhou, and
  Chen]{yang2023large}
Chengrun Yang, Xuezhi Wang, Yifeng Lu, Hanxiao Liu, Quoc~V Le, Denny Zhou, and
  Xinyun Chen.
\newblock Large language models as optimizers.
\newblock \emph{arXiv preprint arXiv:2309.03409}, 2023.

\bibitem[Yang et~al.(2024{\natexlab{b}})Yang, Yu, Zhang, Cao, Xu, Zhang,
  Gonzalez, and Cui]{yang2024buffer}
Ling Yang, Zhaochen Yu, Tianjun Zhang, Shiyi Cao, Minkai Xu, Wentao Zhang,
  Joseph~E Gonzalez, and Bin Cui.
\newblock Buffer of thoughts: Thought-augmented reasoning with large language
  models.
\newblock \emph{Advances in Neural Information Processing Systems},
  2024{\natexlab{b}}.

\bibitem[Yang et~al.(2024{\natexlab{c}})Yang, Yu, Zhang, Xu, Gonzalez, Cui, and
  Yan]{yang2024supercorrect}
Ling Yang, Zhaochen Yu, Tianjun Zhang, Minkai Xu, Joseph~E Gonzalez, Bin Cui,
  and Shuicheng Yan.
\newblock Supercorrect: Supervising and correcting language models with
  error-driven insights.
\newblock \emph{arXiv preprint arXiv:2410.09008}, 2024{\natexlab{c}}.

\bibitem[Yang et~al.(2018)Yang, Qi, Zhang, Bengio, Cohen, Salakhutdinov, and
  Manning]{yang2018hotpotqa}
Zhiyu Yang, Peng Qi, Saizheng Zhang, Yoshua Bengio, William~W Cohen, Ruslan
  Salakhutdinov, and Christopher~D Manning.
\newblock Hotpotqa: A dataset for diverse, explainable multi-hop question
  answering.
\newblock In \emph{Proceedings of the 2018 Conference on Empirical Methods in
  Natural Language Processing (EMNLP)}, 2018.

\bibitem[Yuksekgonul et~al.(2024)Yuksekgonul, Bianchi, Boen, Liu, Huang,
  Guestrin, and Zou]{yuksekgonul2024textgrad}
Mert Yuksekgonul, Federico Bianchi, Joseph Boen, Sheng Liu, Zhi Huang, Carlos
  Guestrin, and James Zou.
\newblock Textgrad: Automatic "differentiation" via text.
\newblock \emph{arXiv preprint arXiv:2406.07496}, 2024.

\bibitem[Zhang et~al.(2024{\natexlab{a}})Zhang, Yue, Sun, Wan, Yu, Fang, Wang,
  and Cheng]{zhang2024g}
Guibin Zhang, Yanwei Yue, Xiangguo Sun, Guancheng Wan, Miao Yu, Junfeng Fang,
  Kun Wang, and Dawei Cheng.
\newblock G-designer: Architecting multi-agent communication topologies via
  graph neural networks.
\newblock \emph{arXiv preprint arXiv:2410.11782}, 2024{\natexlab{a}}.

\bibitem[Zhang et~al.(2024{\natexlab{b}})Zhang, Xiang, Yu, Teng, Chen, Chen,
  Zhuge, Cheng, Hong, Wang, et~al.]{zhang2024aflow}
Jiayi Zhang, Jinyu Xiang, Zhaoyang Yu, Fengwei Teng, Xionghui Chen, Jiaqi Chen,
  Mingchen Zhuge, Xin Cheng, Sirui Hong, Jinlin Wang, et~al.
\newblock Aflow: Automating agentic workflow generation.
\newblock \emph{arXiv preprint arXiv:2410.10762}, 2024{\natexlab{b}}.

\bibitem[Zhong et~al.(2024)Zhong, Wang, Xu, Liu, Ding, Du, and
  Tao]{zhong2024achieving}
Qihuang Zhong, Kang Wang, Ziyang Xu, Juhua Liu, Liang Ding, Bo~Du, and Dacheng
  Tao.
\newblock Achieving> 97\% on gsm8k: Deeply understanding the problems makes
  llms perfect reasoners.
\newblock \emph{arXiv preprint arXiv:2404.14963}, 2024.

\bibitem[Zhou et~al.(2024)Zhou, Ou, Ding, Li, Wu, Wang, Chen, Wang, Xu, Zhang,
  et~al.]{zhou2024symbolic}
Wangchunshu Zhou, Yixin Ou, Shengwei Ding, Long Li, Jialong Wu, Tiannan Wang,
  Jiamin Chen, Shuai Wang, Xiaohua Xu, Ningyu Zhang, et~al.
\newblock Symbolic learning enables self-evolving agents.
\newblock \emph{arXiv preprint arXiv:2406.18532}, 2024.

\bibitem[Zhuge et~al.(2023)Zhuge, Liu, Faccio, Ashley, Csord{\'a}s,
  Gopalakrishnan, Hamdi, Hammoud, Herrmann, Irie, et~al.]{zhuge2023mindstorms}
Mingchen Zhuge, Haozhe Liu, Francesco Faccio, Dylan~R Ashley, R{\'o}bert
  Csord{\'a}s, Anand Gopalakrishnan, Abdullah Hamdi, Hasan Abed Al~Kader
  Hammoud, Vincent Herrmann, Kazuki Irie, et~al.
\newblock Mindstorms in natural language-based societies of mind.
\newblock \emph{arXiv preprint arXiv:2305.17066}, 2023.

\end{thebibliography}

\newpage
\appendix
\onecolumn
\section{Appendix}

\subsection{Proofs}

To help prove Theorem~\ref{deriscoreDPO}, we need the following lemma.

\begin{lemma}
\label{derigxsig}
$F(x) = x \sigma(-a x + b)$ is strictly monotonically increasing with $x > 0$ if $a\le1/x$.
\end{lemma}

\begin{proof}
Given $a\le 1/x$, we have
\begin{align*}
F'(x) &=  \sigma(-a x + b) - a x \sigma(-a x + b) (1 - \sigma(-a x + b))\\
&= \sigma(-a x + b) [1 - a x(1 - \sigma(-a x + b)) ] > 0,
\end{align*}
which means $F(x)$ is strictly monotonically increasing.
\end{proof}

Now we are ready to prove Theorem~\ref{deriscoreDPO}.

\begin{proof}[Proof of Theorem~\ref{deriscoreDPO}]
Through straightforward calculation, we have 
\begin{align*}
I(z) &= \frac{\partial}{\partial r_z} \mathbb{E}_{(w, l) \sim P^{\star}} \left[ \log \sigma \big(r^{\star}_{w} - r^{\star}_{l}\big) \cdot \mathds{1}_{z \in \{w, l\}} \right]\\
& = \frac{\partial}{\partial r_z} \mathbb{E}_{(w, l) \sim P} \left[ d(s_w, s_l) \log \sigma \big(r^{\star}_{w} - r^{\star}_{l}\big) \cdot \mathds{1}_{z \in \{w, l\}} \right]\\
&= \mathbb{E}_{(w, l) \sim P} \Big[ d(s_w, s_l) \sigma (r^{\star}_{l} - r^{\star}_{w}) \big( f(s_w) \mathds{1}_{w = z}  - (1 - f(s_l)) \mathds{1}_{l = z} \big) \Big]\\
&= \underbrace{\mathbb{E}_{(w, l) \sim P} \Big[ \sigma (r^{\star}_{l} - f(s_z) r_{z}) f(s_z) d(s_z, s_l)  \mathds{1}_{w = z} \Big]}_{I_w(z)} - \underbrace{\mathbb{E}_{(w, l) \sim P} \Big[ \sigma ((1 - f(s_z)) r_z - r^{\star}_{w}) (1 - f(s_z)) d(s_w, s_z)  \mathds{1}_{l = z} \Big]}_{I_l(z)}.
\end{align*}
By Lemma~\ref{derigxsig} and condition $ -(1 - f(s_z))^{-1}\le r_z \le f^{-1}(s_z) $, we have $I_w(z)$ is strictly monotonically increasing with $s_z$, while $I_l(z)$ is strictly monotonically decreasing with $s_z$, which implies that $I(z)$ is strictly monotonically increasing with $s_z$.
\end{proof}

\newpage

\subsection{Detailed ScoreFlow Methods}
\label{appdetailedscoreflowmethods}

\subsubsection{Generator Prompt}

To generate the workflow for each problem, we use the following prompt:

\begin{tcolorbox}[
  colback=blue!5!white,
  colframe=gray,
  title=Generator prompt,
  fonttitle=\bfseries\footnotesize,
  sharp corners,
  breakable,
  listing engine=listings,
  listing options={
    language=Python,
    basicstyle=\ttfamily\footnotesize,
    breaklines=true
  }
]
\ttfamily
PROMPT = """
Your objective is to output a workflow graph, based on the following template: 
\{template\}
Here's an introduction to operators you can use: (these are all you can use, do not create new operators)
\{operator\_introductions\}
We have the task input as follow.
\{task\}
You need to notice:
Ensure your graph is based on the given template and is correct to avoid runtime failures. Do NOT import the modules operator and create, which have already been automatically imported. Do not load the operators not provided.
Introducing multiple appropriate operators at appropriate points can enhance performance. Consider Python's loops (for, list comprehensions) to generate multiple solutions to ensemble. Consider logical and control flow (IF-ELSE, loops) for a more enhanced graphical representation.
The graph complexity may corelate with the task complexity. Complex graphs may yield better results, but insufficient information transmission can omit the solution.
Your output graph must be optimized and different from the given template graph. Do not output graph without modification!
Your output graph can not contain any information of the given task due to project requirement. All the information of this problem will be given as input for operators and other agents will execute this workflow.
Only output the optimized graph (remember to add <graph> and </graph>, and the output can not contain specific information of the given task due to project requirement).
Here is the optimized graph:
"""
\end{tcolorbox}

The workflow template we provide is designed for a single agent to guide the generator in adhering to a predefined structure and minimizing runtime execution errors. The operator instructions serve as comprehensive descriptions of the permitted operators available for use. The task presented to the generator requires it to select appropriate operators and produce a workflow that is both well-structured and adaptive to the given task. Our execution model incorporates greater sophistication compared to open-source generators. Consequently, we include a detailed task analysis to guide the generator and ensure that it does not embed specific task-related information directly into the workflow. Additional requirements can be incorporated into the prompt to control the workflow generation.

The template we use imposes no prior knowledge of which structure to build or which operator to choose. We list the templates we used as follows:

\begin{tcolorbox}[colback=blue!5!white, colframe=gray, 
    title=Template for Question Answering, fonttitle=\bfseries\footnotesize, 
    sharp corners, parbox=false, breakable]
\begin{lstlisting}
async def run_workflow(self):
    """
    This is a workflow graph.
    """
    solution = await self.answer_generate()
    
    return solution
\end{lstlisting}
\end{tcolorbox}

\begin{tcolorbox}[colback=blue!5!white, colframe=gray, 
    title=Template for Math Problem, fonttitle=\bfseries\footnotesize, 
    sharp corners, parbox=false, breakable]
\begin{lstlisting}
async def run_workflow(self):
    """
    This is a workflow graph.
    """
    solution = await self.custom(instruction="Can you solve this problem by breaking it down into detailed steps and explaining the reasoning behind each step?")

    return solution
\end{lstlisting}
\end{tcolorbox}

\begin{tcolorbox}[colback=blue!5!white, colframe=gray, 
    title=Template for Coding Problem, fonttitle=\bfseries\footnotesize, 
    sharp corners, parbox=false, breakable]
\begin{lstlisting}
async def run_workflow(self):
    """
    This is a workflow graph.
    """
    solution = await self.code_generate(instruction="Can you analyze this problem step by step and generate the code?")
    
    return solution
\end{lstlisting}
\end{tcolorbox}

\newpage

\subsubsection{Operator Utilized}

For mathematical problems, we utilize the following custom operators: the \textbf{Custom Operator}, which generates outputs based on a fixed input problem and modifiable instructions; the \textbf{Programmer}, which automatically writes and executes Python code to derive and return the final solution based on the given problem description and analysis; the \textbf{Ensemble Operator}, which evaluates all generated solutions and selects the best one from the solution list; and the \textbf{Reviewer}, which reviews previous solutions to refine and regenerate improved solutions.

For question-answering problems, we utilize the following operators: the Custom Operator, the \textbf{AnswerGenerate Operator}, which directly generates answers, including the reasoning process, for the given problem; the Ensemble Operator, which evaluates all generated answers and selects the best one; and the Reviewer, which reviews and refines previous answers to produce improved solutions. 
For coding tasks, we utilize the following operators: the \textbf{CustomCodeGenerate Operator}, which generates code based on customized input instructions; the Ensemble Operator, which evaluates multiple code solutions and selects the best one; and the \textbf{Test Operator}, which refines the input solution by testing it against public test cases. We also include an Answer Extractor Agent after the final response in each generated workflow to eliminate redundant information and ensure concise and precise evaluation. The design of these operators is based on the Aflow framework \citep{zhang2024aflow}.

The following are the introductions to the operators we used.
\begin{tcolorbox}[colback=blue!5!white, colframe=gray, 
    title=Introductions to Operators, fonttitle=\bfseries\footnotesize, 
    sharp corners, parbox=false, breakable]
\begin{lstlisting}
1. Custom:
Usage: Generates anything based on fixed input problem and modifiable instruction.
Format MUST follow: custom(instruction: str) -> str
You can modify the instruction prompt. The output can serve as the input of next operators or the final output.
2. CustomCodeGenerate:
Usage: Generates code based on customized input instruction.
Format MUST follow: code_generate(instruction: str) -> str
The instruction should encourage operator to think step by step, do not add the specific information of the task into the input instruction.
The output can serve as the input of next operators or the final output.
3. AnswerGenerate:
Usage: Directly generate answer (including thought) to the given problem.
Format MUST follow: answer_generate() -> str
For example:
solution = await self.answer_generate()
The output can serve as the input of next operators or the final output.
4. Programmer:
Usage: Automatically writes, executes Python code, and returns the final solution based on the provided problem description and analysis.
Format MUST follow: programmer(analysis: str = 'None') -> str
The input analysis can be outputs of some other operators, for exmaple:
program_solution = await self.programmer(analysis=solution)
The output can serve as the input of next operators or the final output.
5. ScEnsemble:
Usage: Evaluate every solutions, then select the best solution in the solution list.
Format MUST follow: sc_ensemble(solutions: List[str]) -> str
You can ensemble few solutions, for example:
ensembled_solution = await self.sc_ensemble(solutions=solution_list)
The output can serve as the input of next operators or the final output.
6. Review:
Usage: Given previous solution, Review operator reviews the previous solution to regenerate the solution.
Format MUST follow: review(pre_solution: str) -> str
pre_solution should be solution from previous operator, for example
rev_solution = await self.review(pre_solution=pre_solution)
The output can serve as the input of next operators or the final output.
7. Test:
Usage: Modify the input solution by testing the solution using public test cases.
Format MUST follow: test(solution: str) -> str
tested_solution = await self.test(solution=pre_solution)
\end{lstlisting}
\end{tcolorbox}

\subsubsection{The Detailed Algorithm}

We have the detailed algorithm in Algorithm~\ref{alg:scoreflow}.

\begin{algorithm}[H]
\caption{\textbf{ScoreFlow}}
\label{alg:scoreflow}
\begin{algorithmic}[1]
\STATE \textbf{Input:} 
\STATE \quad 1) A set of problems $D = \{q_1, q_2, \dots, q_N\}.$
\STATE \quad 2) A workflow generator $G$ parameterized by $\theta.$
\STATE \quad 3) Number of iterations $M.$
\STATE \quad 4) Number of workflows generated per problem in optimization: $k.$
\STATE \quad 5) Number of preference samples generated in each iteration: $S.$
\STATE \quad 5) Executor LLM for evaluation.

\STATE \textbf{Initialize:} Generator parameters $\theta$.

\FOR{$t = 1$ to $M$ or not converged}
    \STATE \textbf{Collect preference data:}
    \FOR{each problem $q \in D$}
        \FOR{$i = 1$ to $k$}
            \REPEAT
                \STATE Use Generator $G_{\theta}$ to generate workflow $g_i(q)$ for problem $q$.
            \UNTIL{Condition $C^{\star}$ holds for $g_i(q)$}
            \STATE Collect the workflow $g_i(q)$ for problem $q$.
        \ENDFOR
        \STATE Obtain $k$ candidate workflows $\{g_i(q)\}_{i=1}^k$ using $g$.
        \STATE Evaluate each $g_i(q)$ with the executor LLM to obtain score $s_i \in [0,1]$.
        \STATE Construct preference pairs 
        \[
          D_q \!=\! \Big\{\!( (q, g_i(q)), (q, g_j(q)) ) \mid s_i > s_j \!\Big\}.
        \]
    \ENDFOR
    \STATE Aggregate preferences $D_{pre} \leftarrow \bigcup_{q \in D} D_q$ (Denote its raw distribution as $P$).
    \STATE \textbf{Update the generator via Score-DPO:}
    \FOR{$j = 1$ to $S$ or not converged}
        \STATE Generate preference samples $(w, l)$ by sampling distribution $P^{\star}(w, l) \propto P(w, l) d(w, l)$.
        \STATE Calculate $r_w^{\star} = f(s_w) r_w$ and $r_l^{\star} = ( 1- f(s_l)) r_l$.
        \STATE Obtain loss function:
        \[
        \mathcal{L}_{\text{Score-DPO}} = -  \log \sigma (r^{\star}_{w} - r^{\star}_{l}) 
        \]
        \STATE Fine-tune $G_{\theta}$ to obtain updated parameters \( \theta \leftarrow \theta - \eta \nabla_{\theta} \mathcal{L}_{\text{Score-DPO}} \), where \( \eta \) is the learning rate.
    \ENDFOR
\ENDFOR

\STATE \textbf{Output:} Trained generator $G_{\theta}$.
\end{algorithmic}
\end{algorithm}

The condition \( C^{\star} \) ensures quality control of our workflow by guaranteeing the absence of runtime errors and adherence to execution time limits. We select $M = 3$ and $k = 8$. We set \( S = 2000 \) for all datasets except for the smallest one, HumanEval, where we use \( S = 600 \).

\newpage

\subsection{Additional Experiment Results}

\subsubsection{Sampled workflows}

The following are some examples of the generated workflows.

\begin{tcolorbox}[colback=blue!5!white, colframe=gray, 
    title=Example Workflow 1 for Question Answering, fonttitle=\bfseries\footnotesize, 
    sharp corners, parbox=false, breakable]
\begin{lstlisting}
async def run_workflow(self):
    """
    This is a workflow graph.
    """
    solution = await self.answer_generate()
    
    solution = await self.review(solution)
    
    solution_list = [solution]
    for _ in range(3):
        solution = await self.custom_2(
            instruction="Can you solve this problem by breaking it down into detailed steps and explaining the reasoning behind each step?"
        )
        solution_list.append(solution)
    
    ensembled_solution = await self.sc_ensemble(solution_list)

    return ensembled_solution
\end{lstlisting}
\end{tcolorbox}

\begin{tcolorbox}[colback=blue!5!white, colframe=gray, 
    title=Example Workflow 2 for Question Answering, fonttitle=\bfseries\footnotesize, 
    sharp corners, parbox=false, breakable]
\begin{lstlisting}
async def run_workflow(self):
    """
    This is a workflow graph.
    """
    solution_1 = await self.custom_1(instruction="Can you break down the problem into smaller steps?")
    
    solution_2 = await self.answer_generate_1()
    
    solution_3 = await self.sc_ensemble_1(solutions=[solution_1, solution_2])
    
    solution_4 = await self.review_1(pre_solution=solution_3)
    
    solution_5 = await self.custom_2(instruction="Can you explain the reasoning behind each step?")
    
    solution_6 = await self.sc_ensemble_2(solutions=[solution_4, solution_5])
    
    solution_7 = await self.review_2(pre_solution=solution_6)
    
    return solution_7
\end{lstlisting}
\end{tcolorbox}

\begin{tcolorbox}[colback=blue!5!white, colframe=gray, 
    title=Example Workflow 3 for Question Answering, fonttitle=\bfseries\footnotesize, 
    sharp corners, parbox=false, breakable]
\begin{lstlisting}
async def run_workflow(self):
    """
    This is a workflow graph.
    """
    solution_list = []
    for _ in range(3):
        solution = await self.custom_1(instruction="Can you break down the problem into smaller steps?")
        solution_list.append(solution)
    
    ensembled_solution = await self.sc_ensemble(solutions=solution_list)
    
    rev_solution = await self.review(pre_solution=ensembled_solution)

    return rev_solution
\end{lstlisting}
\end{tcolorbox}

\begin{tcolorbox}[colback=blue!5!white, colframe=gray, 
    title=Example Workflow 1 for Math Problem, fonttitle=\bfseries\footnotesize, 
    sharp corners, parbox=false, breakable]
\begin{lstlisting}
async def run_workflow(self):
    """
    This is a workflow graph.
    """
    solution_1 = await self.custom_1(instruction="Can you break down the problem into smaller steps?")
    
    solution_2 = await self.custom_2(instruction="Can you explain the reasoning behind each step?")
    
    solution_list = [solution_1, solution_2]
    ensembled_solution = await self.sc_ensemble(solutions=solution_list)
    
    analysis = ensembled_solution
    program_solution = await self.programmer(analysis=analysis)
    
    final_solution = await self.review(pre_solution=program_solution)
    
    return final_solution
\end{lstlisting}
\end{tcolorbox}

\begin{tcolorbox}[colback=blue!5!white, colframe=gray, 
    title=Example Workflow 2 for Math Problem, fonttitle=\bfseries\footnotesize, 
    sharp corners, parbox=false, breakable]
\begin{lstlisting}
async def run_workflow(self):
    """
    This is a workflow graph.
    """
    solution1 = await self.custom1(instruction="Can you break down the problem into smaller steps and explain each step clearly?")
    
    solution2 = await self.custom1(instruction="Can you explain the problem and provide a step-by-step solution?")
    
    solution3 = await self.custom2(instruction="Can you describe the problem and provide a detailed solution?")
    
    solutions = [solution1, solution2, solution3]
    ensembled_solution = await self.sc_ensemble(solutions=solutions)
    
    program_solution = await self.programmer(analysis=ensembled_solution)
    
    final_solution = await self.review(pre_solution=program_solution)

    return final_solution
\end{lstlisting}
\end{tcolorbox}

\begin{tcolorbox}[colback=blue!5!white, colframe=gray, 
    title=Example Workflow 3 for Math Problem, fonttitle=\bfseries\footnotesize, 
    sharp corners, parbox=false, breakable]
\begin{lstlisting}
async def run_workflow(self):
    """
    This is a workflow graph.
    """
    solution_1 = await self.custom_1(instruction="Can you break down the problem into smaller steps?")
    
    solution_2 = await self.custom_2(instruction="Can you explain the solution in a clear and concise manner?")
    
    analysis = solution_1 + solution_2
    program_solution = await self.programmer(analysis=analysis)
    
    solutions = [solution_1, solution_2, program_solution]
    ensembled_solution = await self.sc_ensemble(solutions=solutions)
    
    final_solution = await self.review(pre_solution=ensembled_solution)
    
    return final_solution
\end{lstlisting}
\end{tcolorbox}

\begin{tcolorbox}[colback=blue!5!white, colframe=gray, 
    title=Example Workflow 1 for Coding Problem, fonttitle=\bfseries\footnotesize, 
    sharp corners, parbox=false, breakable]
\begin{lstlisting}
async def run_workflow(self):
    """
    This is a workflow graph.
    """
    solution = await self.code_generate(instruction="Can you analyze the problem step by step and generate the code?")
    solution_list = [solution]
    
    for i in range(3):
        solution_list.append(await self.code_generate(instruction="Can you think step by step and generate the code?"))
    
    ensembled_solution = await self.sc_ensemble(solutions=solution_list)
    
    tested_solution = await self.test(solution=ensembled_solution)
    
    return tested_solution
\end{lstlisting}
\end{tcolorbox}

\begin{tcolorbox}[colback=blue!5!white, colframe=gray, 
    title=Example Workflow 2 for Coding Problem, fonttitle=\bfseries\footnotesize, 
    sharp corners, parbox=false, breakable]
\begin{lstlisting}
async def run_workflow(self):
    """
    This is a workflow graph.
    """
    solution_list = []
    for i in range(5):
        solution = await self.code_generate(instruction="Can you generate a code to solve a problem?")
        solution_list.append(solution)
    
    ensembled_solution = await self.sc_ensemble(solutions=solution_list)

    tested_solution = await self.test(solution=ensembled_solution)
    return tested_solution
\end{lstlisting}
\end{tcolorbox}

\newpage

\subsubsection{Case study on $d(x, y)$}
\label{appalbd}

In this section, we analyze the performance of Score-DPO under various formulations of the function \( d(x, y) = (x - y)^{\alpha} \). DPO is a special case where \(\alpha = 0\) when ignoring $f(x)$. Our results demonstrate that increasing \(\alpha\) leads to improved performance, as higher values of \(\alpha\) upweight the more deterministic preference pairs, thereby reducing variance and error in the collected preference data. However, when \(\alpha\) is taken to an extreme, such as \(\alpha = 100\), the performance deteriorates significantly. This decline occurs because excessively prioritizing only the most deterministic pairs effectively disregards a substantial portion of the preference pair data. Moreover, DPO inherently tends to favor out-of-distribution (unseen) responses or data \citep{xu2024dpo}. By omitting less deterministic pairs, the model loses valuable information, which adversely impacts its ability to generalize effectively.

\begin{table}[H]
\scriptsize
\renewcommand\tabcolsep{3.2pt}
\renewcommand\arraystretch{1.2}
\small
\setlength{\abovecaptionskip}{0.1cm}
\setlength{\belowcaptionskip}{-0.2cm}
\centering
\caption{Case studies for $d(x, y)$ across MBPP, DROP, and MATH datasets. Values represent averaged solve rates on test set.}
\label{tabped(x,y)}
\scalebox{1.0}{%
\begin{tabular}{l|c|c|c}
\specialrule{1.2pt}{0pt}{0pt} 
\textbf{Method} 
& \textbf{MBPP} 
& \textbf{DROP} 
& \textbf{MATH} \\
\hline
\hline
DPO    
& 82.9 
& 84.2 
& 60.4 \\

$d(x, y) = (x - y)^2$   
& 83.5 
& 85.4 
& 61.4 \\

$d(x, y) = (x - y)^3$   
& 84.7 
& 86.2 
& 64.4 \\

$d(x, y) = (x - y)^{100}$   
& 80.1 
& 85.1 
& 59.2 \\
\specialrule{1.2pt}{0pt}{0pt} 
\end{tabular}%
}
\vspace{-3mm}
\end{table}

\subsubsection{Experimental Validation for Condition in Theorem~\ref{deriscoreDPO}}
\label{appevfcit}

The sufficient condition provided in Theorem~\ref{deriscoreDPO} is $|r_z| \le 1$. In this section, we evaluate this condition by estimating the probability that it holds during the optimization process. Our analysis reveals that this condition is satisfied with probabilities of $99.8\%$, $82.2\%$, and $91.2\%$ for the MATH, DROP, and MBPP datasets, respectively. On average, the condition is upheld with a probability of \textbf{$91.1\%$} across three different tasks in our experiments, demonstrating its robustness across diverse datasets and substantiating its practical applicability in guiding optimization under varied scenarios.

\begin{figure}[H]
    \centering
    \includegraphics[width=0.7\textwidth]{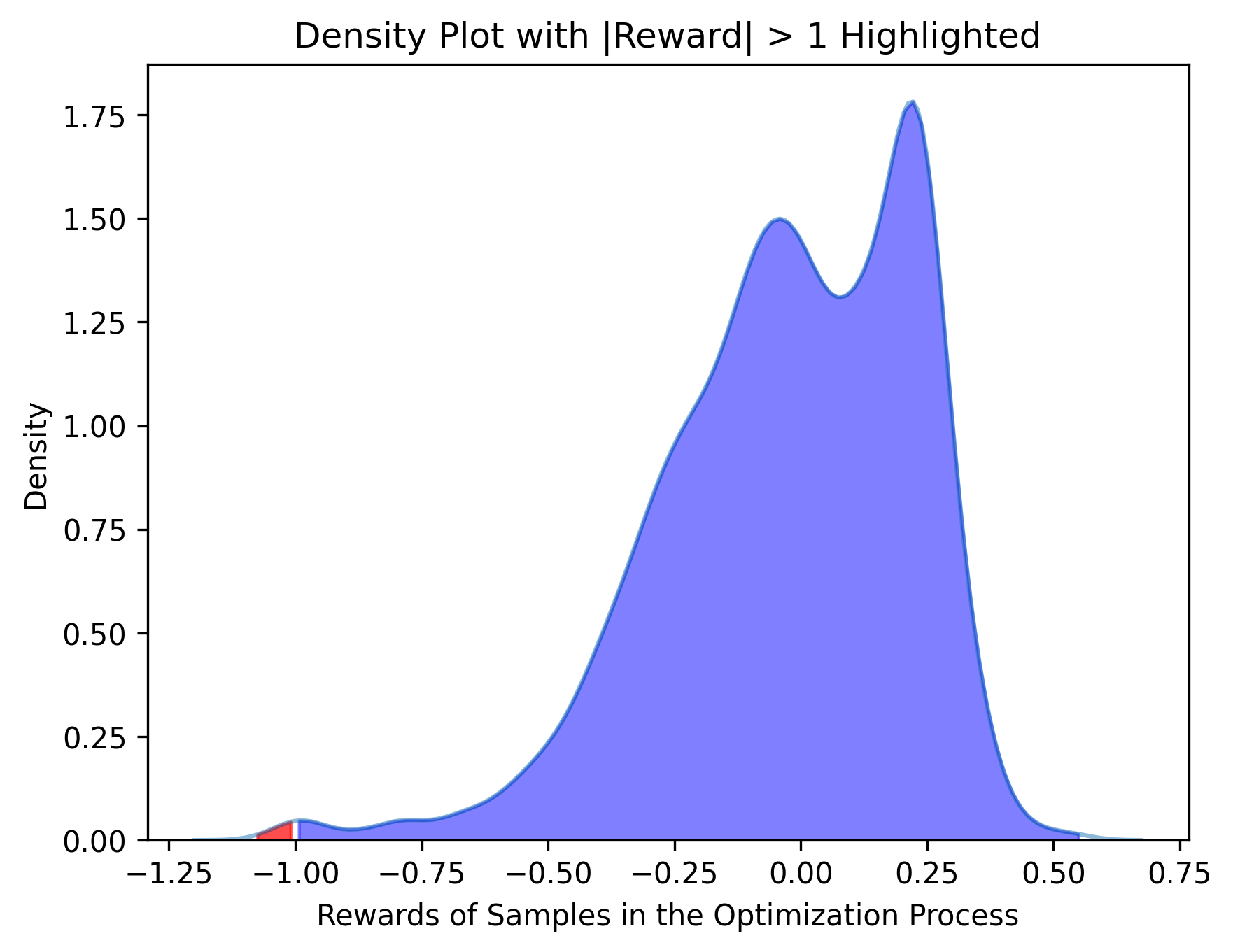}
    \caption{The distribution of sample implicit reward during optimization process before convergence (MATH).}
    \label{figmath_proportion}
\end{figure}

\begin{figure}[H]
    \centering
    \includegraphics[width=0.7\textwidth]{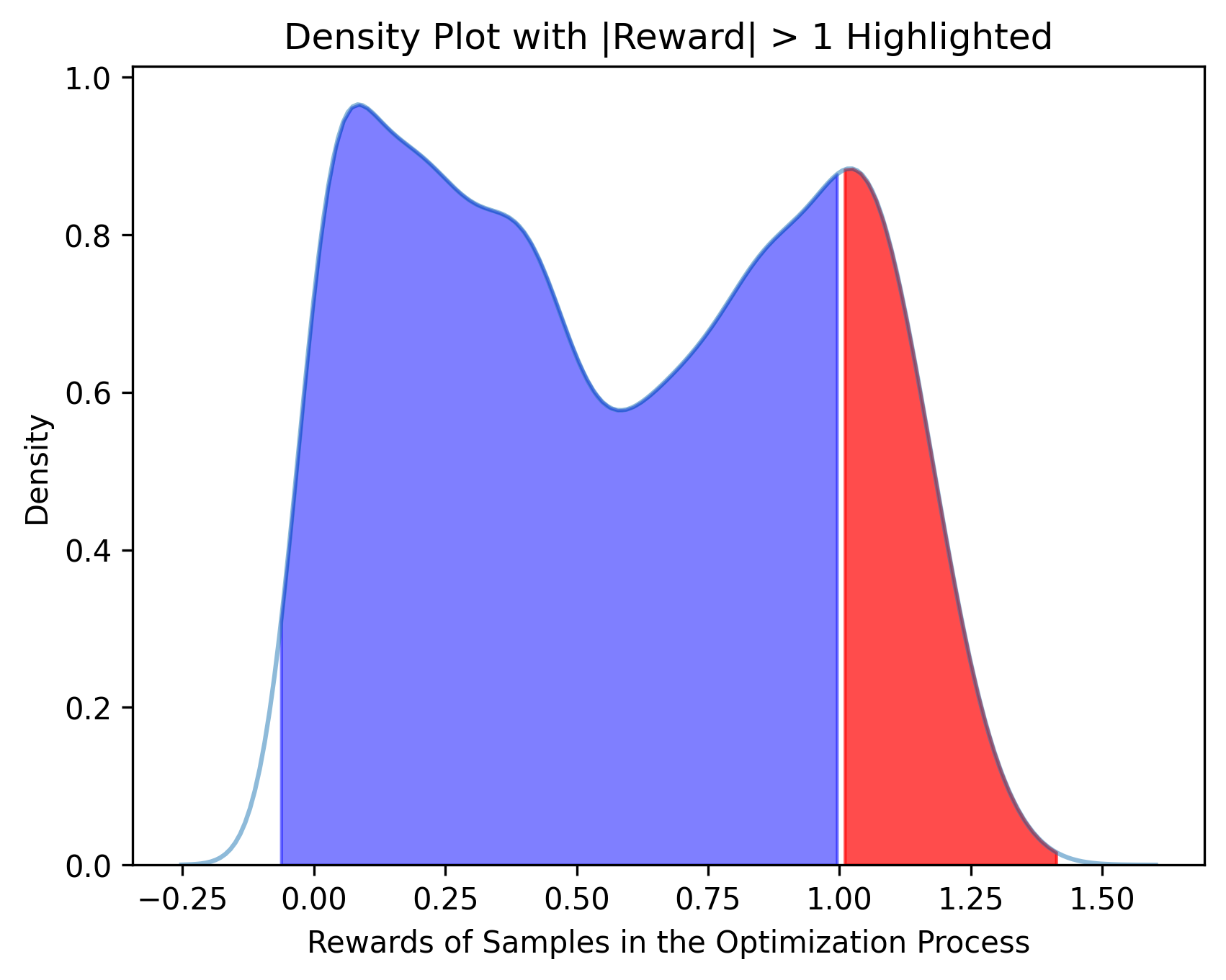}
    \caption{The distribution of sample implicit reward during optimization process before convergence (DROP).}
    \label{figdrop_proportion}
\end{figure}

\begin{figure}[H]
    \centering
    \includegraphics[width=0.7\textwidth]{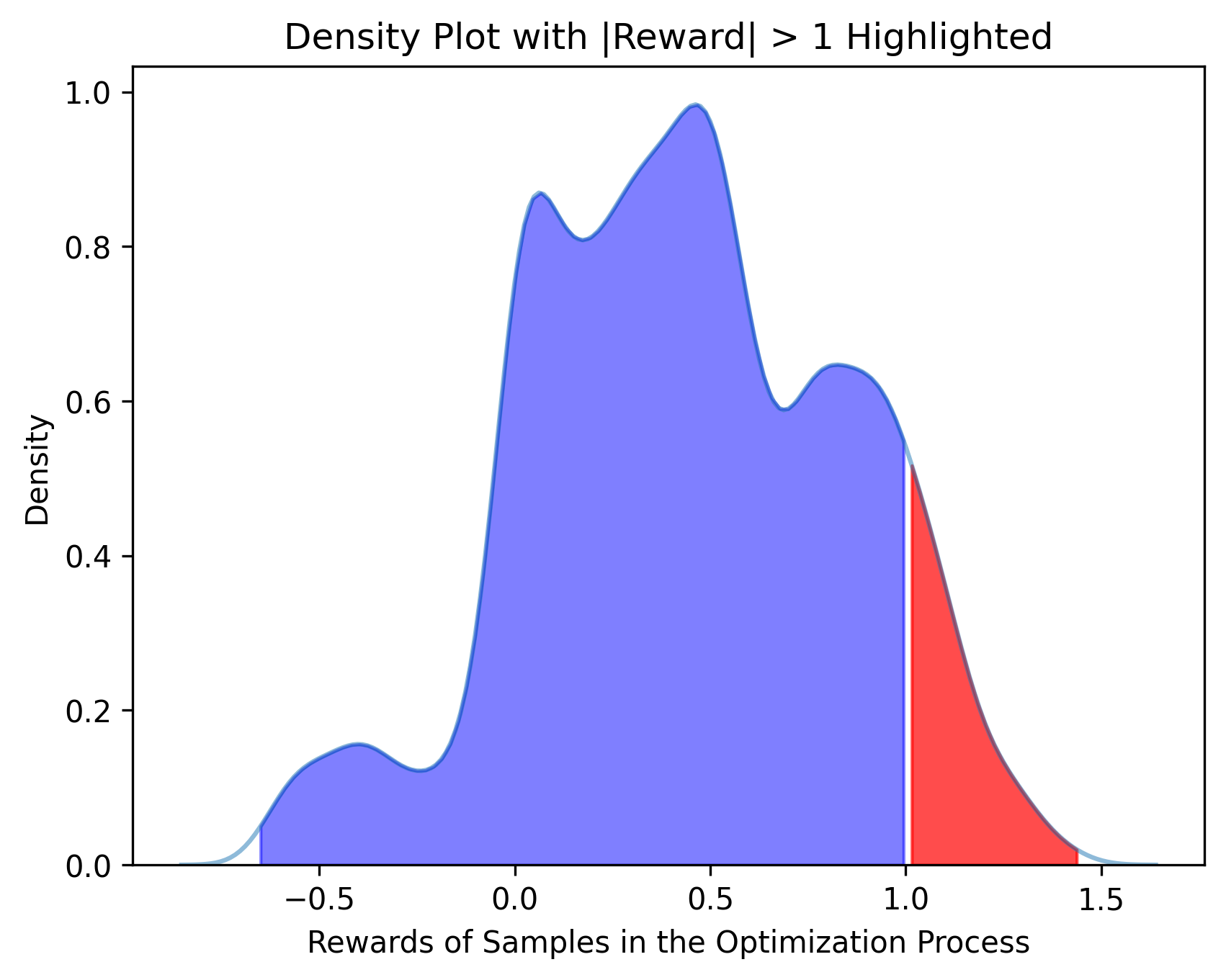}
    \caption{The distribution of sample implicit reward during optimization process before convergence (MBPP).}
    \label{figmbpp_proportion}
\end{figure}

\subsubsection{Detailed Cost Data}

\begin{table}[H]
\scriptsize
\renewcommand\tabcolsep{3.2pt}
\renewcommand\arraystretch{1.2}
\small
\setlength{\abovecaptionskip}{0.1cm}
\setlength{\belowcaptionskip}{-0.2cm}
\centering
\caption{The detailed cost value ($\$$) in Figure~\ref{figcostoptimization} (optimization process on test data).}
\label{tab:cost_analysisoptimization}
\scalebox{1.0}{%
\begin{tabular}{l|c|c|c}
\specialrule{1.2pt}{0pt}{0pt} 
\textbf{Method} 
& \textbf{GPT-4o-mini} 
& \textbf{DeepSeek-V3} 
& \textbf{DeepSeek-coder} \\
\hline
\hline
Ours    
& 2.2570 
& 1.4124 
& 1.3966 \\

Aflow    
& 4.6081 
& 2.9160 
& 2.8664 \\
\specialrule{1.2pt}{0pt}{0pt} 
\end{tabular}%
}
\vspace{-3mm}
\end{table}

\begin{table}[H]
\scriptsize
\renewcommand\tabcolsep{3.2pt}
\renewcommand\arraystretch{1.2}
\small
\setlength{\abovecaptionskip}{0.1cm}
\setlength{\belowcaptionskip}{-0.2cm}
\centering
\caption{The detailed cost value ($\$$) in Figure~\ref{figcostinference} (inference process on test data).}
\label{tab:cost_analysisinference}
\scalebox{1.0}{%
\begin{tabular}{l|c|c|c|c}
\specialrule{1.2pt}{0pt}{0pt} 
\textbf{Method} 
& \textbf{GPT-4o-mini} 
& \textbf{GPT-4o} 
& \textbf{DeepSeek-V3} 
& \textbf{DeepSeek-coder} \\
\hline
\hline
Ours    
& 0.2281 
& 5.1549 
& 0.1336 
& 0.1246 \\

Aflow    
& 0.2021 
& 3.9549 
& 0.1229 
& 0.1253 \\

IO    
& 0.0483 
& 1.2281 
& 0.0251 
& 0.0301 \\

CoT    
& 0.0536 
& 1.9688 
& 0.0300 
& 0.0473 \\

CoT SC    
& 0.3155 
& 7.3738 
& 0.1825 
& 0.1817 \\

MP    
& 0.3497 
& 9.5230 
& 0.2265 
& 0.2392 \\

MPD    
& 0.3789 
& 10.8530 
& 0.2425 
& 0.2276 \\

SR    
& 0.1243 
& 1.8651 
& 0.0728 
& 0.0699 \\
\specialrule{1.2pt}{0pt}{0pt} 
\end{tabular}%
}
\vspace{-3mm}
\end{table}

\end{document}